\documentclass{article}

\usepackage[main, final]{neurips_2026}

\usepackage[utf8]{inputenc} 
\usepackage{caption}
\usepackage[T1]{fontenc}    
\usepackage{hyperref}       
\usepackage{url}            
\usepackage{booktabs}       
\usepackage{amsfonts}       
\usepackage{nicefrac}       
\usepackage{wrapfig}
\usepackage{microtype}      
\usepackage{xcolor}         
\usepackage{xspace}
\usepackage{algorithm}
\usepackage{algpseudocode}
\usepackage{amsmath}
\usepackage{subcaption}
\usepackage{graphicx}
\usepackage{amsthm}
\usepackage{array}
\usepackage{placeins}
\usepackage{float}
\usepackage{multirow}
\hypersetup{
    colorlinks=true,
    citecolor=[rgb]{0.0, 0.3, 0.7},
    urlcolor=[rgb]{0.0, 0.3, 0.7},
    linkcolor=[rgb]{0.0, 0.3, 0.7},
}

\newcommand{\method}{\textsc{OP-Mix}\xspace}

\usepackage{pifont}
\newcommand{\cmark}{\textcolor{green!70!black}{\ding{51}}}
\newcommand{\xmark}{\textcolor{red}{\ding{55}}}

\newcommand{\D}{\mathcal{D}}

\newcommand{\LM}{\text{LM}}

\newtheorem{proposition}{Proposition}
\newtheorem{assumption}{Assumption}
\theoremstyle{definition}

\theoremstyle{plain}
\newtheorem*{remark}{Remark}

\newtheorem{lemma}{Lemma}[section]
\newtheorem{corollary}[lemma]{Corollary}
\newcounter{savedtheorem}

\newcommand{\ie}{\textit{i.e.},\xspace}
\newcommand{\eg}{\textit{e.g.},\xspace}

\title{Always Learning, Always Mixing: \\Efficient and Simple Data Mixing All the Time}

\author{
\textbf{Michael Y. Hu}$^{1}$ \quad
\textbf{Apurva Gandhi}$^{2}$ \quad \\
\textbf{Kyunghyun Cho}$^{1}$ \quad
\textbf{Tal Linzen}$^{1}$ \quad
\textbf{Pratyusha Sharma}$^{1,3}$ \vspace{0.5em} \\
$^{1}$New York University \quad
$^{2}$Carnegie Mellon University\quad
$^{3}$Microsoft \\
\texttt{\{michael.hu, kyunghyun.cho, linzen, ps5795\}@nyu.edu} \\
\texttt{apurvag@andrew.cmu.edu}
}

\begin{document}

\maketitle


\begin{abstract}

Data mixing decides how to combine different sources or types of data and is a consequential problem throughout language model training. 
In pretraining, data composition is a key determinant of model quality; in continual learning and adaptation, it governs what is retained and acquired. 
Yet existing data mixing methods address only one phase of this lifecycle at a time: 
some require smaller proxy models tied to a single training phase, others assume a fixed domain set, and continual learning lacks principled guidance altogether. 
\emph{We argue that data mixing is fundamentally an online decision making problem---one that recurs throughout training and demands a single, unified solution.} 
We introduce \method (On-Policy Mix), a data mixing algorithm that operates across 
pretraining, midtraining, and instruction tuning.
Our main insight is that candidate data mixtures can be cheaply simulated by interpolating between low-rank adapters trained directly on the current model, eliminating separate proxy models and ensuring the search is always grounded in the model's actual learning dynamics.
Across the language model training lifecycle,
\method consistently finds near-optimal mixtures while using a fraction of the compute of the baselines. In pretraining, \method improves upon training without mixing by 6.3\% in average perplexity. 
For continual learning, \method matches the performance of both retraining and on-policy distillation while using 66\% and 95\% less overall compute, respectively. 
\method suggests a different view of language model training: not a sequence of distinct phases, but a single continuous process of learning from data.

\end{abstract}

\begin{figure}[t]
    \centering
    \includegraphics[width=0.85\linewidth]{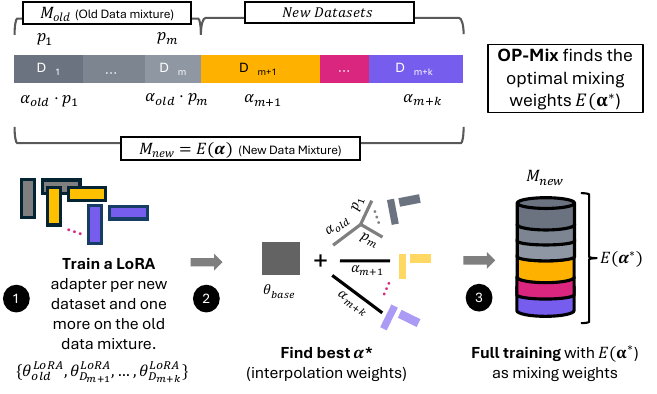}
    \caption{\textbf{Overview of \method.} \method aims to cheaply estimate optimal data mixing ratios in a continual setting. \textbf{(1)} Train a lightweight LoRA adapter on new domains to estimate future performance. \textbf{(2)} Interpolate adapters to simulate different data mixtures without retraining and then estimate the optimal mixture ratio. \textbf{(3)} Train the base model with the computed mixture.}
    \label{fig:overview}
\end{figure}

\begin{figure}[t]
    \centering
    \includegraphics[width=\linewidth]{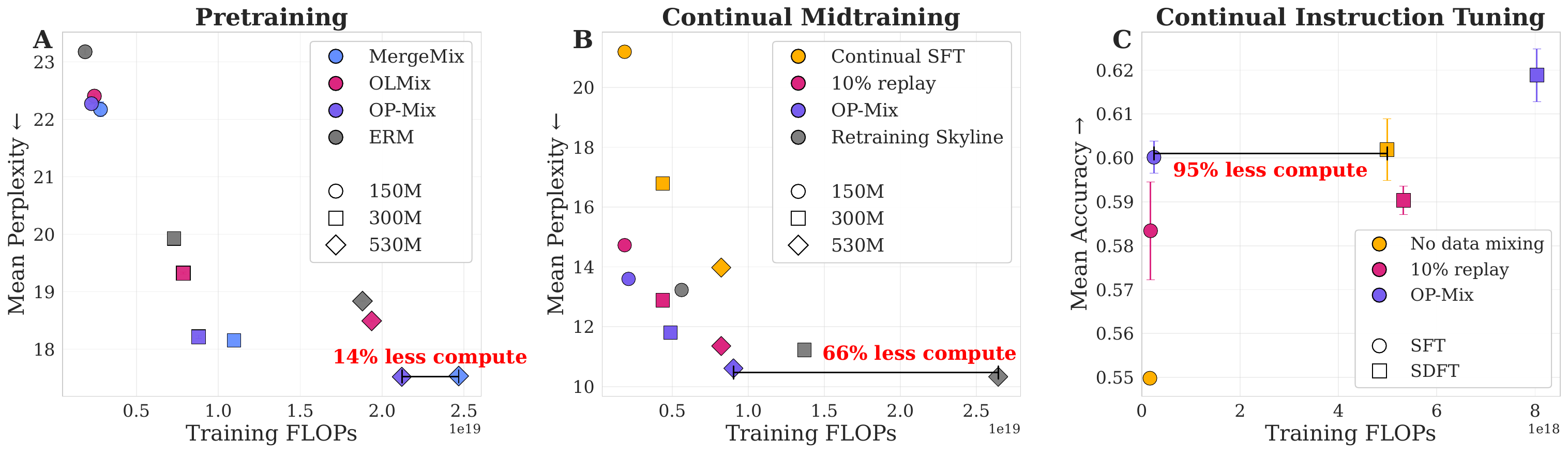}
    \caption{\textbf{\method (purple) Pareto-dominates the performance-efficiency frontier} when tested alongside baselines across pretraining, continual midtraining, and continual instruction tuning.}
    \label{fig:efficiency}
\end{figure}

\section{Introduction}

Language models are trained on carefully curated data mixtures, yet the science of constructing the right mixture remains nascent. 
The dominant approach---training small proxy models on candidate mixtures and extrapolating to full training---is combinatorially expensive and scales poorly as the number of domains grows \citep{ye2025data,liu2025regmix,chen2025olmix}. 
Furthermore, most data mixing approaches are specialized towards pretraining and assume a fixed domain set \citep{chen2025aioli,fan2024doge,jiang2025adaptive,xie2023doremi,skillit}: in practice, available training domains evolve continuously as new tasks are defined, new corpora are collected, and new capabilities are prioritized. This induces a natural continual learning problem, where the goal is to incorporate new data without catastrophically forgetting what the model has already learned. We ask:
\begin{center}
    \emph{What is the right data mix, and how do we find it efficiently as the data itself keeps changing?}
\end{center}

We propose \method (\textbf{O}n-\textbf{P}olicy Mix), an algorithm that estimates optimal data mixtures by combining two insights. First, rather than train separate proxy models for each candidate data mixture, \method trains a single low-rank adapter (LoRA, \cite{hu2022lora}) per data domain directly from the current model, keeping the proxy model \textit{on-policy} with the model being trained---\ie reflective of its current state. Second, it uses linear interpolation between LoRAs as a proxy for the loss surface of full data mixing, following recent works \citep{wang2026mergemix,tao2025mergetomix}.
This bypasses the need to retrain proxies for every different data mix ratio, escaping the combinatorial explosion of training runs.
These two insights allow \method to search over data mixtures with minimal additional compute, no separate proxy models, and natural accommodation to new domains: when a new dataset arrives, we simply train another LoRA and re-fit the mixture.

We evaluate \method across three stages of the language model lifecycle---pretraining \citep{radford2019language,devlin2019bert}, continual midtraining \citep{olmo2,liu2026midtraining}, and continual instruction tuning \citep{wei2022finetuned}---and find that our single algorithm suffices for all three. In pretraining, \method improves over no data mixing by 6.3\% in average perplexity and matches the best data mixing baseline's performance while using 14\% less compute. In continual midtraining, \method achieves the performance of \textit{full retraining} at a fraction of the cost. 
Finally, in continual instruction tuning, \method composes with on-policy self-distillation 
\citep{shenfeld2026selfdistillation,lu2025thinking,zhao2026self}, 
yielding further gains without modifications to either algorithm. 
Our contributions are as follows:
\begin{enumerate}
    \item \textbf{The first universal data mixing algorithm: }\method is the first data mixing algorithm that both expands to new data domains and simulates candidate mixtures without separate proxy models. This enables \method to continually mix data even as the data evolves, overcoming the need for a different algorithm at each phase of the training pipeline (\S \ref{sec:method}). 
    \item \textbf{State-of-the-art in performance per compute across the entire training lifecycle:} \method is pareto-optimal in pretraining, continual midtraining, and continual instruction tuning, demonstrating that phase-specific algorithms are unnecessary. (\S \ref{sec:experiments}).
    \item \textbf{\method enables continual learning, matching on-policy distillation with 95\% less compute:}
    Applied atop standard SFT during continual instruction tuning, \method recovers the gains of self-distillation finetuning (SDFT, \citet{shenfeld2026selfdistillation}) at a fraction of the cost. Combining \method with SDFT also yields further gains, suggesting that data mixing can be an independent axis of improvement from training objective (\S \ref{sec:main-results}).
\end{enumerate}

Code: \url{https://github.com/michahu/on-policy-mix}

\section{Background: Data Mixing and Its Limitations}
\label{sec:background}

Let $\D = \{D_1, D_2, \dots, D_m\}$ be a set of $m$ data domains, where domain $D_i$ has $N_i$ tokens. A data mixture is a probability vector $p \in \triangle^{m-1}$, where training on $R$ total tokens uses $p_i \cdot R$ tokens from domain $D_i$. We denote a language model of $S$ parameters trained for $R$ tokens on mixture $p$ as $\LM(S, R, p)$ and measure its performance on downstream task $j \in [J]$ as $f_j(\LM(S, R, p))$. We assume the training objective is to minimize a weighted sum $F = \sum_j w_j \cdot f_j(\LM(S, R, p))$, where weights $w_j$ are user-specified. 
Here, metrics intended to be maximized (\eg accuracy) are negated. 

\paragraph{Batch continual learning.} During training, we may periodically receive $k$ new datasets $D_{m+1}, \dots, D_{m+k}$, in which case the updated domain set becomes $\D \cup \{D_{m+1}, \dots, D_{m+k}\}$. For example, these $k$ new datasets may be instruction fine-tuning datasets, introduced after pretraining. We may then aim to minimize the loss across both pretraining and instruction tuning datasets.

\paragraph{Data mixing.} Data mixing algorithms automate the process of finding the mixture $p$ that minimizes $F$. The core idea in most data mixing algorithms is fitting a simple model $\hat{f_i}(p)$ that predicts the future performance $f_i$ as a function of the performance on the current data mixture $p$ (see \citet{chen2025aioli} for review). One can then minimize $\hat{f_i}(p)$ to estimate an optimal mixture.

\begin{wraptable}{r}{0.6\textwidth}
\vspace{-1em}
\centering
\begin{tabular}{@{}l l p{1.55cm} p{1.3cm}@{}}
\toprule
\multicolumn{2}{@{}l}{Algorithm} & Adapts to new data? & No proxy models? \\
\midrule
ADO      & \citep{jiang2025adaptive} & \centering\xmark & \centering\arraybackslash\cmark \\
Aioli    & \citep{chen2025aioli}     & \centering\xmark & \centering\arraybackslash\cmark \\
DoGE     & \citep{fan2024doge}       & \centering\xmark & \centering\arraybackslash\cmark \\
GRAPE    & \citep{fan2025grape}      & \centering\xmark & \centering\arraybackslash\cmark \\
OLMix    & \citep{chen2025olmix}      & \centering\cmark & \centering\arraybackslash\xmark \\
RegMix   & \citep{liu2025regmix}      & \centering\xmark & \centering\arraybackslash\xmark \\
MergeMix & \citep{wang2026mergemix}   & \centering\xmark & \centering\arraybackslash\cmark \\
\midrule
\method  & (ours)                     & \centering\cmark & \centering\arraybackslash\cmark \\
\bottomrule
\end{tabular}
\caption{\textbf{\method is the only method that expands the data mixture to new data while not using separate proxy models.} The combination of these two features allows \method to be deployed across the language model lifecycle.}
\label{tab:data-mixing-algorithms}
\vspace{-1em}
\end{wraptable}

Previous work has shown that future performance is well-predicted by a log-linear parametric form: $\hat{f_i}(p) = c_i + \exp{(A_i^\top p_i)}$, where $c_i \in \mathbb{R}$ and $A_i \in \mathbb{R}^m$ \citep{ye2025data,chen2025aioli,chen2025olmix}. Data mixing algorithms then aim to estimate such a scaling law as cheaply as possible. A common technique is to fit the scaling law by randomly sampling mixtures from the probability simplex and training proxy models with fewer parameters $S' \ll S$ and less data $R' \ll R$ to approximate the full model's performance on a data mixture: $f_i(\LM(S, R, p)) \approx f_i(\LM(S', R', p))$ \citep{liu2025regmix,chen2025olmix,ye2025data}.

A single data mixing algorithm that spans all of LM training is desirable for both practical reasons (less complexity and phase-specific tuning) and conceptual ones: pretraining, midtraining, and finetuning are not fundamentally different problems for data mixing.
However, two issues limit existing algorithms from operating across this lifecycle. 
First, most data mixing algorithms, being targeted towards pretraining, do not expand their data mixtures. It follows that these algorithms cannot be applied to the continual learning setting, and yet language model training induces a natural continual learning problem from phase to phase. 
Second, data mixing methods that rely on separate proxy models are defunct after pretraining, as open-source model releases typically do not come with a matching small-model proxy \citep{gemma3,llama3,qwen3}.
Moreover, separate smaller proxy models have been shown to yield suboptimal mixtures for the target model, as they diverge from the base model's dynamics at scale \citep{jiang2025adaptive,chen2025olmix} and the number of proxies explodes combinatorially with the number of datasets.

\begin{figure}[t]
    \centering
    \includegraphics[width=\linewidth]{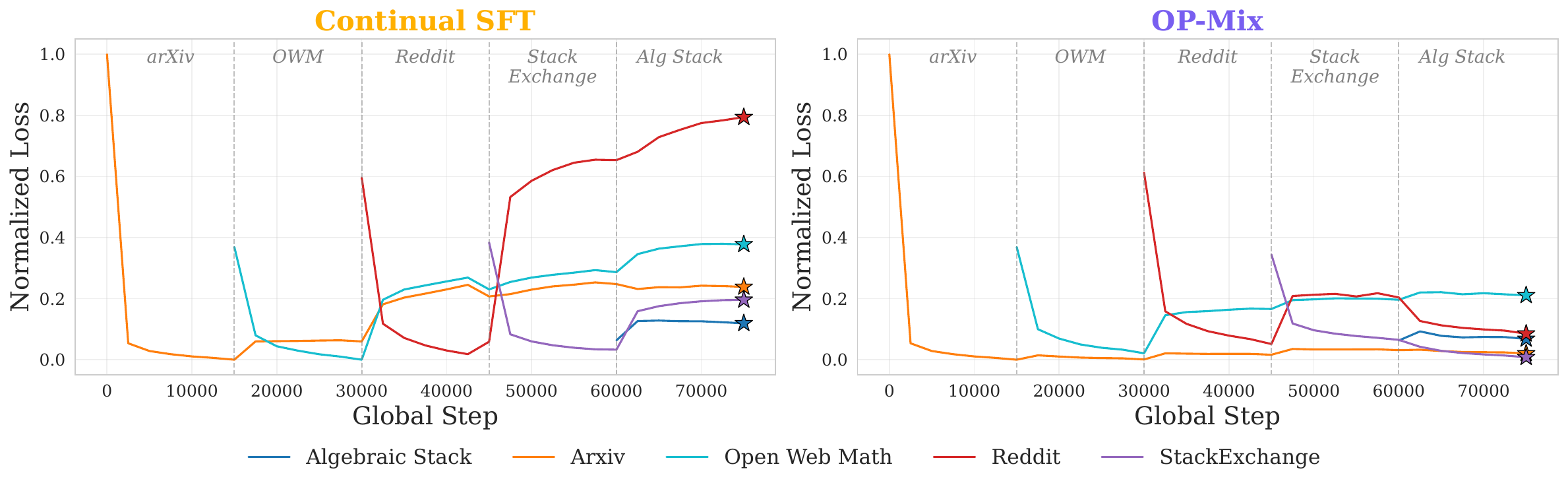}
    \caption{\textbf{\method enables continual learning:} For a 530M parameter model, \method mitigates forgetting 27\% better on average and 71\% better on Reddit than Continual SFT with WSD-S learning schedule, a method specifically designed for continual learning \citep{wen2025understanding}.}
    \label{fig:continual-learning-dynamics}
\end{figure}

\section{\method: On-Policy Data Mixing}
\label{sec:method}

In this work, we propose \method, a data mixing algorithm that works effectively for any stage of language model training by using on-policy proxies instead of separate proxies and efficiently expanding data mixtures. 
\emph{On-policy} here means that the proxy is built from the model being trained, rather than a separately initialized model whose learning dynamics may diverge from its target.
\method uses Low-Rank Adaptation (LoRA, \cite{hu2022lora}) to cheaply estimate the performance of full training. LoRA reduces the necessary compute for testing new data mixtures while being tied to the base model and circumvents the ambiguities of creating separate proxy models later in training. 

To simulate new data mixtures without performing additional training runs, we interpolate LoRA weights, as inspired by \citet{wang2026mergemix} and \citet{tao2025mergetomix}. This allows us to train one LoRA per data domain and estimate the effect of mixing domains post-hoc using only forward passes. 
We also expand the data mixture when new domains arrive, taking inspiration from pretraining mixture reuse in \citet{chen2025olmix}. In each stage, instead of retraining a new LoRA for \textit{every} previously seen domain, we train a single ``old'' adapter $\theta_{D_{\text{old}}}^{\text{LoRA}}$, keeping probabilities of old domains constant and only adjusting the ratio between the old mixture and incoming new domains.

\paragraph{\method (Algorithm~\ref{alg:opmix}).}

In the \textbf{continual setting}, when $K$ new domains $D_{m+i}, i \in [K]$ arrive, we train a single LoRA adapter per domain $D_{m+i}$, starting from the current model. This gives us $\theta_{D_{m+i}}^{\text{LoRA}}$, a cheap approximation of what full finetuning on $D_{m+i}$ would produce. We also train $\theta_{D_{\text{old}}}^{\text{LoRA}}$ on the old data to approximate continued training on $D_{\text{old}}$. Next, we evaluate linear interpolation merges of $\theta_{D_{m+1}}^{\text{LoRA}}, \dots, \theta_{D_{m+K}}^{\text{LoRA}}$ and $\theta_{D_{\text{old}}}^{\text{LoRA}}$. We sample interpolation points in the $K$-simplex $\triangle^{K}$, and each  interpolation point simulates a different mixing ratio between old and new data without additional training. We then fit a regression model to these evaluations, producing a smooth loss surface over the interpolation path (Algorithm~\ref{alg:opmix}, lines 7--12). Finally, we minimize over this surface to obtain $\alpha^\star$, the tradeoff between old and new data, distribute the resulting weight across all datasets, and do the final training run. 
See Figure~\ref{fig:overview} for a visual overview.

For \textbf{pretraining}, we begin with a warmup phase in which every document is sampled with equal probability (empirical risk minimization). After warmup, we reintroduce each dataset as new domains to adjust the data mixture. In §\ref{sec:main-results}, we set warmup to 20\% of the overall token budget.

\begin{algorithm}[t]
  \caption{\method\ (single continual learning step)}
  \label{alg:opmix}
  \begin{algorithmic}[1]
    \State \textbf{Input:} Base model $\theta_{\text{base}}$; previous domains $\{D_1,\dots,D_m\}$ with mixture
    $p_{t-1}\in\triangle^{m-1}$; new domains $\{D_{m+1},\dots,D_{m+K}\}$;
    mixture prior $\mu\in\triangle^{m+K-1}$; search iterations $P$;
    regularization strength towards prior $\lambda$.

    \State Train LoRA adapter $\theta^{\text{LoRA}}_{\text{old}}$ on the mixture $p_{t-1}$,
    starting from $\theta_{\text{base}}$.

    \For{$k \in [K]$}
      \State Train LoRA adapter $\theta^{\text{LoRA}}_{D_{m+k}}$ on $D_{m+k}$,
      starting from $\theta_{\text{base}}$.
    \EndFor

    \State Define the mixture expansion $E:\triangle^{K} \to \triangle^{m+K-1}$ by
    \[
      E(\boldsymbol{\alpha})_i \;=\;
      \begin{cases}
        \alpha_{\text{old}}\, p_{t-1}(D_i) & i \le m \\[2pt]
        \alpha_i & i > m,
      \end{cases}
      \qquad
      \boldsymbol{\alpha} = (\alpha_{\text{old}}, \alpha_{m+1}, \dots, \alpha_{m+K}).
    \]

    \For{$p \in [P]$}
      \State Sample $\boldsymbol{\alpha}_p \sim \triangle^{K}$.
      \State Form the merged adapter
        $\theta^{\text{LoRA}}_{\boldsymbol{\alpha}_p} \;\leftarrow\;
        \alpha_{\text{old}}\,\theta^{\text{LoRA}}_{\text{old}}
        + \sum_{k=1}^{K}\alpha_{m+k}\,\theta^{\text{LoRA}}_{D_{m+k}}.$
      \State Evaluate per-domain loss
        $y_{p,j} = f_j\!\bigl(\theta^{\text{LoRA}}_{\boldsymbol{\alpha}_p}\bigr)
        \qquad \text{for } j=1,\dots,N.$
    \EndFor

    \State Fit log-linear regressors $\hat g_j(\boldsymbol{\alpha})$ to
    $\{(\boldsymbol{\alpha}_p, y_{p,j})\}_{p=1}^{P}$.

    \State Solve the regularized mix optimization. \Comment{cvxpy \citep{diamond2016cvxpy}}
    \[
      \boldsymbol{\alpha}^\star \;=\; \arg\!\min_{\boldsymbol{\alpha}\in\triangle^{K}}\;
      \frac{1}{N}\sum_{j=1}^{N}\hat g_j(\boldsymbol{\alpha})
      \;+\; \lambda \, D_{\text{KL}}\!\bigl(E(\boldsymbol{\alpha}) \,\big\|\, \mu\bigr).
    \]

    \State Set the new mixture $p_t \leftarrow E(\boldsymbol{\alpha}^\star)$ and continue
    training $\theta_{\text{base}}$ on $p_t$.

    \State \Return $p_t$ and the fine-tuned model (the next stage's $\theta_{\text{base}}$).
  \end{algorithmic}
\end{algorithm}

\section{Experimental Results}
\label{sec:experiments}

We examine \method in several settings: pretraining, continual midtraining, and continual instruction fine-tuning. In pretraining, we examine \method's ability to find a good pretraining mixture for fixed training corpora. In midtraining \citep{olmo2}, high quality datasets are upweighted relative to the original pretraining data mixture; here we continually finetune a pretrained model ladder from HuggingFace on successive reference datasets. Finally, we apply \method to the continual instruction tuning setting, where a language model is finetuned on successive question-answering datasets, and consider two different objectives: cross-entropy loss and on-policy distillation \citep{shenfeld2026selfdistillation,lu2025thinking,zhao2026self}. Further training details and hyperparameter choices are in Appendix \ref{app:reproducibility}.

\paragraph{Pretraining baselines.}
\textbf{ERM} samples from each data domain with probability proportional to domain size, equivalent to not optimizing the data mixture.
\textbf{MergeMix}~\citep{wang2026mergemix} finetunes independent models on each dataset, merges to simulate mixing, and uses regression to estimate the optimal mixture; we adapt it to pretraining with a 20\% ERM warmup before finetuning (see Appendix \ref{app:pretrain} for more details). It is a natural comparison---essentially \method without data mixture expansion (Algorithm~\ref{alg:opmix}, line 6) and with full finetuning in place of LoRA.
\textbf{OLMix}~\citep{chen2025olmix} trains small proxy models on randomly sampled mixtures over datasets and uses regression to estimate the optimal mixture.

\paragraph{Continual learning baselines.}
\textbf{Continual fine-tuning with WSD-S}~\citep{wen2025understanding} trains on each dataset in succession with no replay of old data, using Weight-Stable-Decay-Simplified, a learning rate schedule designed for continual learning. For simplicity, we use WSD-S for all methods, including \method.
\textbf{10\% data replay} extends the observation of \citet{bethune2025scaling} that having 10\% of finetuning data be pretraining data mitigates catastrophic forgetting; we train with a 1:9 ratio between old and new data.
\textbf{Retraining} (skyline): After training for $K \cdot R$ tokens from $K$ datasets and receiving a $(K{+}1)$th dataset, train again for $(K{+}1) \cdot R$ tokens over all $K{+}1$ datasets. 

To our knowledge, there currently are no adaptive data mixing baselines for continual learning. As noted in \S \ref{sec:background}, existing data mixing methods either operate over fixed data domains or require separate proxy models initialized from scratch. 


\begin{figure}[h]
    \centering
    \includegraphics[width=\linewidth]{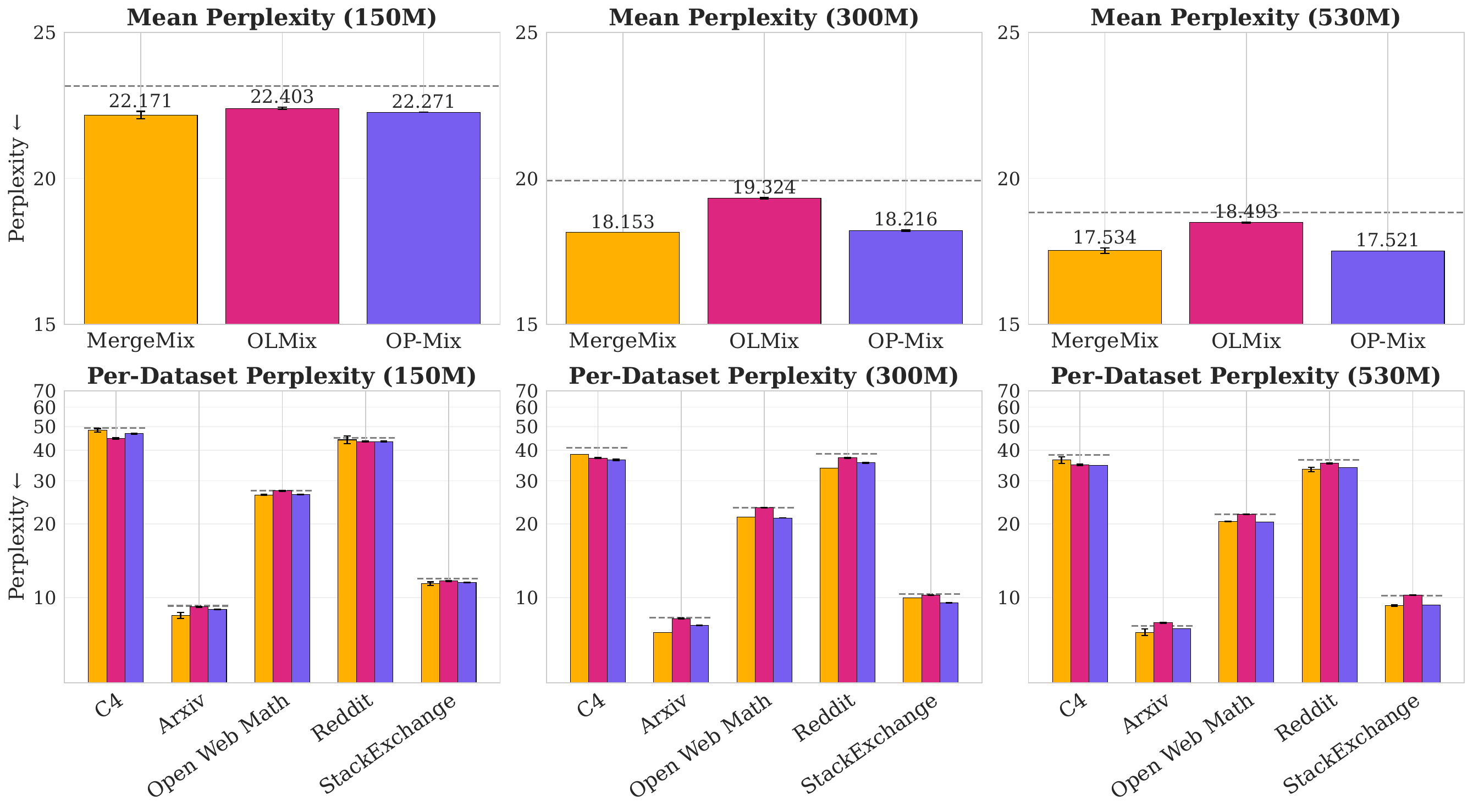}
    \caption{\textbf{Pretraining:} \method outperforms empirical risk minimization (grey line), which samples from all data domains with uniform probability, and beats or matches the performance of other data mixing baselines while being up to 14\% more efficient (Figure \ref{fig:efficiency}).}
    \label{fig:pretraining}
\end{figure}

\subsection{\method Works Across the Language Model Lifecycle}
\label{sec:main-results}

Across pretraining, continual midtraining, and continual instruction tuning, \method achieves state-of-the-art performance (Figures \ref{fig:pretraining}–\ref{fig:continual-instruction}).

\paragraph{Pretraining (Figure \ref{fig:pretraining}).} We pretrain three different model sizes---150M, 300M, and 530M---from the OLMo model ladder \citep{groeneveld-etal-2024-olmo} to Chinchilla-optimal \citep{hoffmann2022training} token counts of 3.2B, 6.5B, and 10.5B, respectively. We construct the pretraining data from 5 data domains: Algebraic Stack, ArXiv, c4, Reddit, and StackExchange \citep{c4,weber2024redpajama}. During evaluation, we measure perplexity on all data domains and compute overall perplexity by a simple unweighted average. Each data domain contains more than 10.5B tokens, so no data mixture trains for more than one epoch on any data domain. 

In pretraining, \method matches the performance of MergeMix using up to 14\% less compute (Figure \ref{fig:efficiency}A) and outperforms OLMix by 5-6\% at every scale. This is consistent with our on-policy hypothesis: \method and MergeMix both build proxies from the model being trained, while OLMix uses a separate proxy whose learning dynamics diverge from the base model. These results are roughly mirrored by the downstream evaluations in Appendix Table \ref{tab:downstream}, where \method is either best or second-best in downstream task performance. Overall, \method consistently outperforms ERM in perplexity and downstream evaluations while being more efficient than MergeMix.

\begin{figure}[t]
    \centering
    \includegraphics[width=\linewidth]{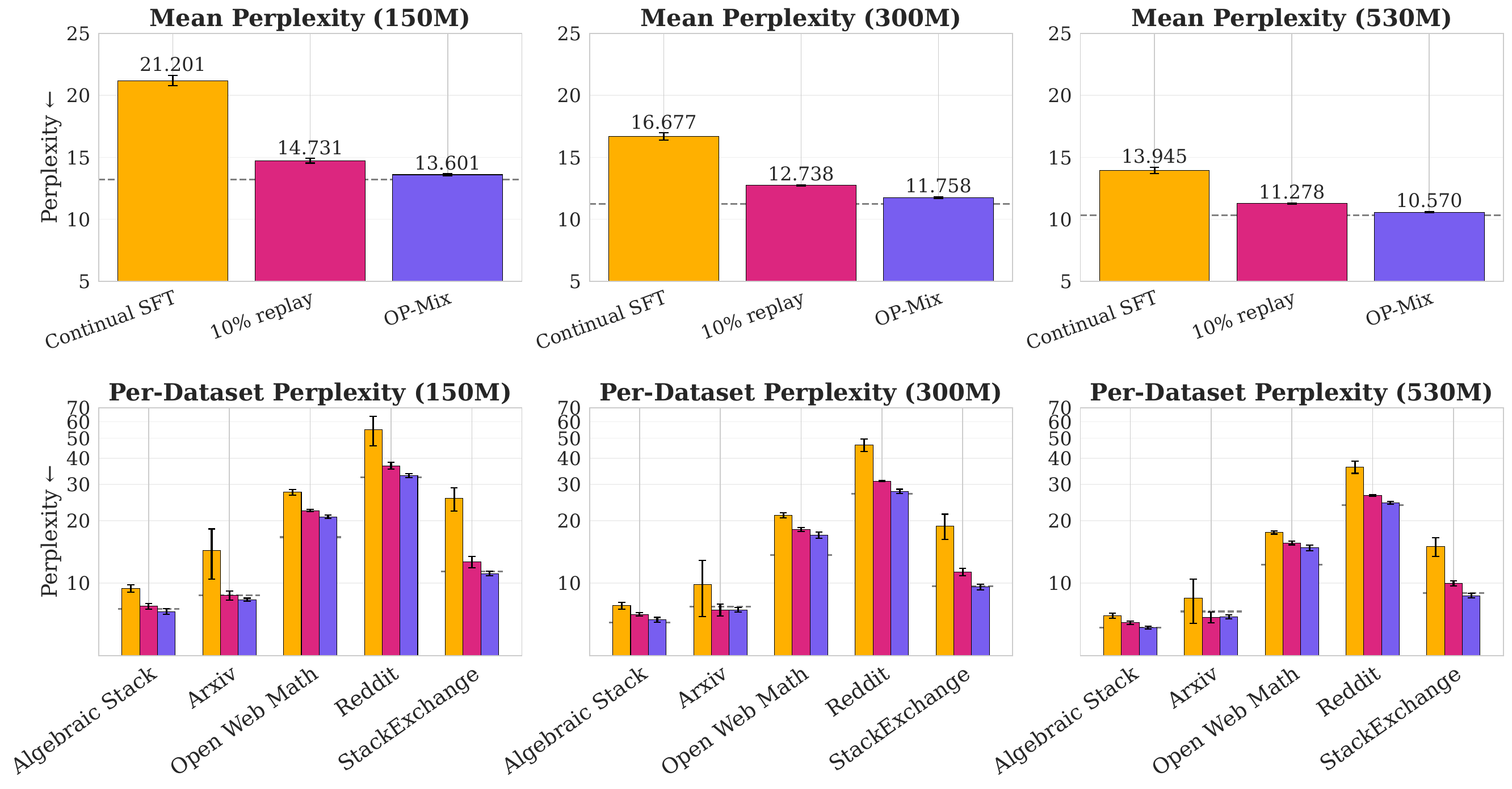}
    \caption{\textbf{Continual midtraining:} \method outperforms other continual learning baselines and is even competitive with full retraining (grey line), despite training on the datasets sequentially.}
    \label{fig:continual-learning}
\end{figure}

\paragraph{Continual Midtraining (Figure \ref{fig:continual-learning}).} In continual midtraining, the user receives a stream of new data domains, emulating the real-life scenario where one updates a base model using new reference datasets. Here, our data mixture must expand per new dataset. In this section, we finetune open-source LMs pretrained on C4 \citep{c4} from the DataDecide model suite \citep{magnusson2025datadecide}. We continually finetune models of parameter counts 150M, 300M, and 530M on Algebraic Stack, ArXiv, Open Web Math, Reddit, and StackExchange in alphabetic order. To account for ordering effects, we cyclically permute the order of the datasets so that each dataset appears once in every order position and train on all five combinations (see Table \ref{tab:ordering} in Appendix).

During continual midtraining, continual SFT suffers severe catastrophic forgetting (Figure \ref{fig:continual-learning-dynamics}), and of the data-mixing methods, \method is best at mitigating it,
nearly matching the performance of full retraining (Figure \ref{fig:vs-retraining}) while using up to 66\% less compute.
In Figure \ref{fig:lora-merge} (Appendix), we also include an ablation that merges in trained LoRAs into the base model using the optimized $\alpha^*$, instead of full finetuning. Although better than Continual SFT, LoRA-Merge is significantly worse than \method, indicating that there are benefits to using LoRA only as a proxy.

\paragraph{Continual Instruction Tuning (Figure \ref{fig:continual-instruction}).} We take the following continual learning task and ordering verbatim from \cite{shenfeld2026selfdistillation}: we continually finetune Qwen2.5-7B-Instruct \citep{yang2024qwen2} on three instruction-following domains---Tool Use (4k examples), Science (1.2k examples), and Medical (10k examples)---introduced one at a time. In addition to standard SFT, we test Self-Distillation Finetuning (SDFT) \citep{shenfeld2026selfdistillation} as the training objective. Performance is measured by mean accuracy across domains. 

\method on top of standard SFT (60.0\%) matches the performance of SDFT (60.2\%) while using 95\% less compute, demonstrating that data mixing alone can recover the gains of a more sophisticated continual learning algorithm. (SDFT uses more compute as it repeatedly generates and distills on its own training data.) The two methods are also synergistic: combining \method with SDFT achieves the best overall performance (61.9\%), suggesting that data mixing and objective modifications are orthogonal axes of improvement.

\subsection{Efficiency: \method Pareto-Dominates on the Performance-Efficiency Frontier}

Figure \ref{fig:efficiency} (on page 2) compares methods by final performance versus total training FLOPs, counting both mixture selection and final training. \method Pareto-dominates across pretraining, continual midtraining, and continual instruction tuning: no baseline achieves better performance at lower compute. \method's reuse of mixtures especially matters in continual midtraining (Figure \ref{fig:efficiency}B), where the cost of naive retraining grows with every new domain. Unlike retraining, \method turns adaptive data mixing into a lightweight operation that can be repeated whenever new data arrives.

\begin{figure}[ht]
    \centering
    \includegraphics[width=0.8\linewidth]{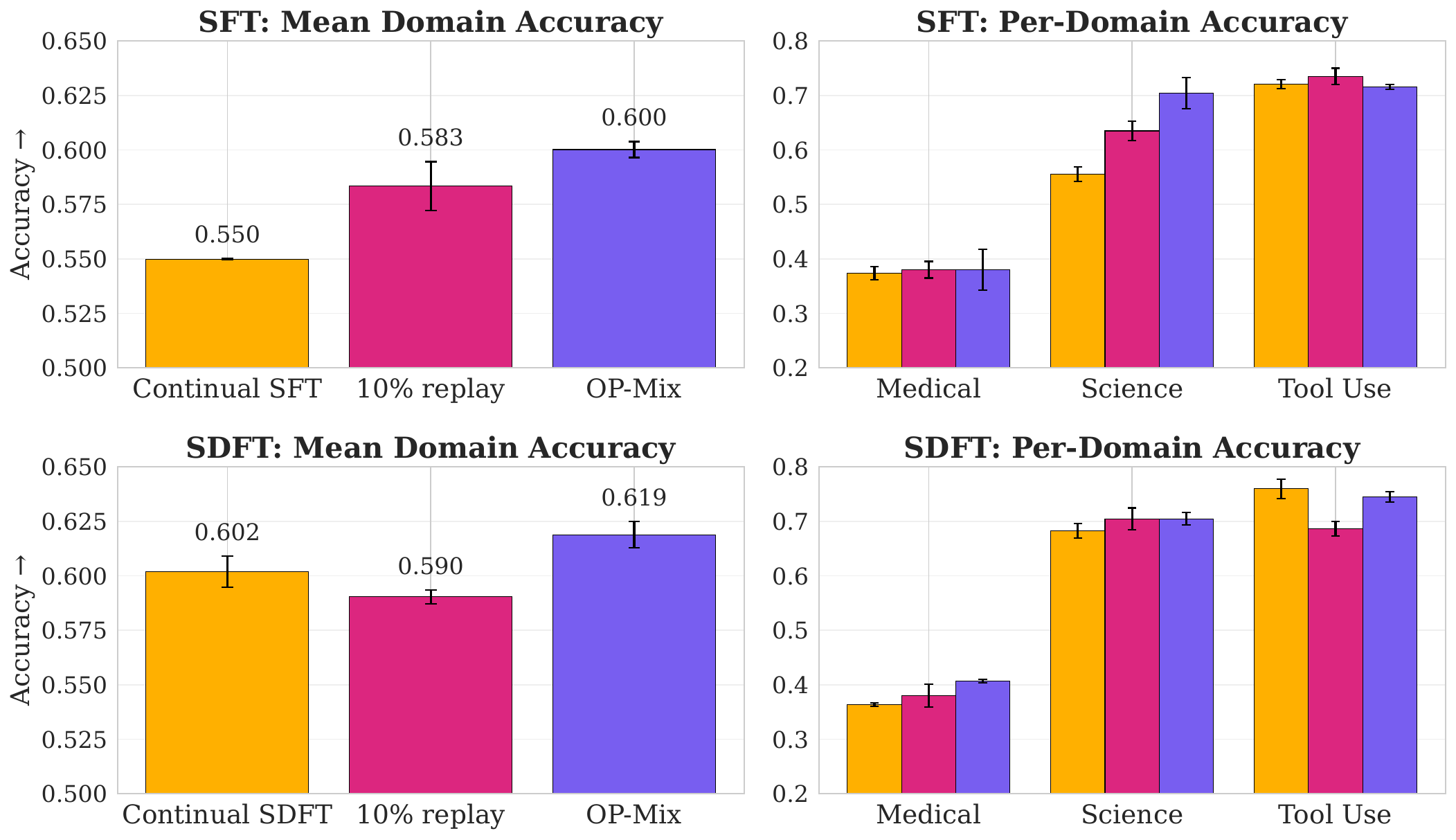}
    \caption{\textbf{Continual instruction tuning: } \method works across cross-entropy and KL distillation objectives, improving the performance of both supervised finetuning and self-distillation fine tuning.}
    \label{fig:continual-instruction}
\end{figure}

\begin{figure}[h]
    \centering
    \includegraphics[width=\linewidth]{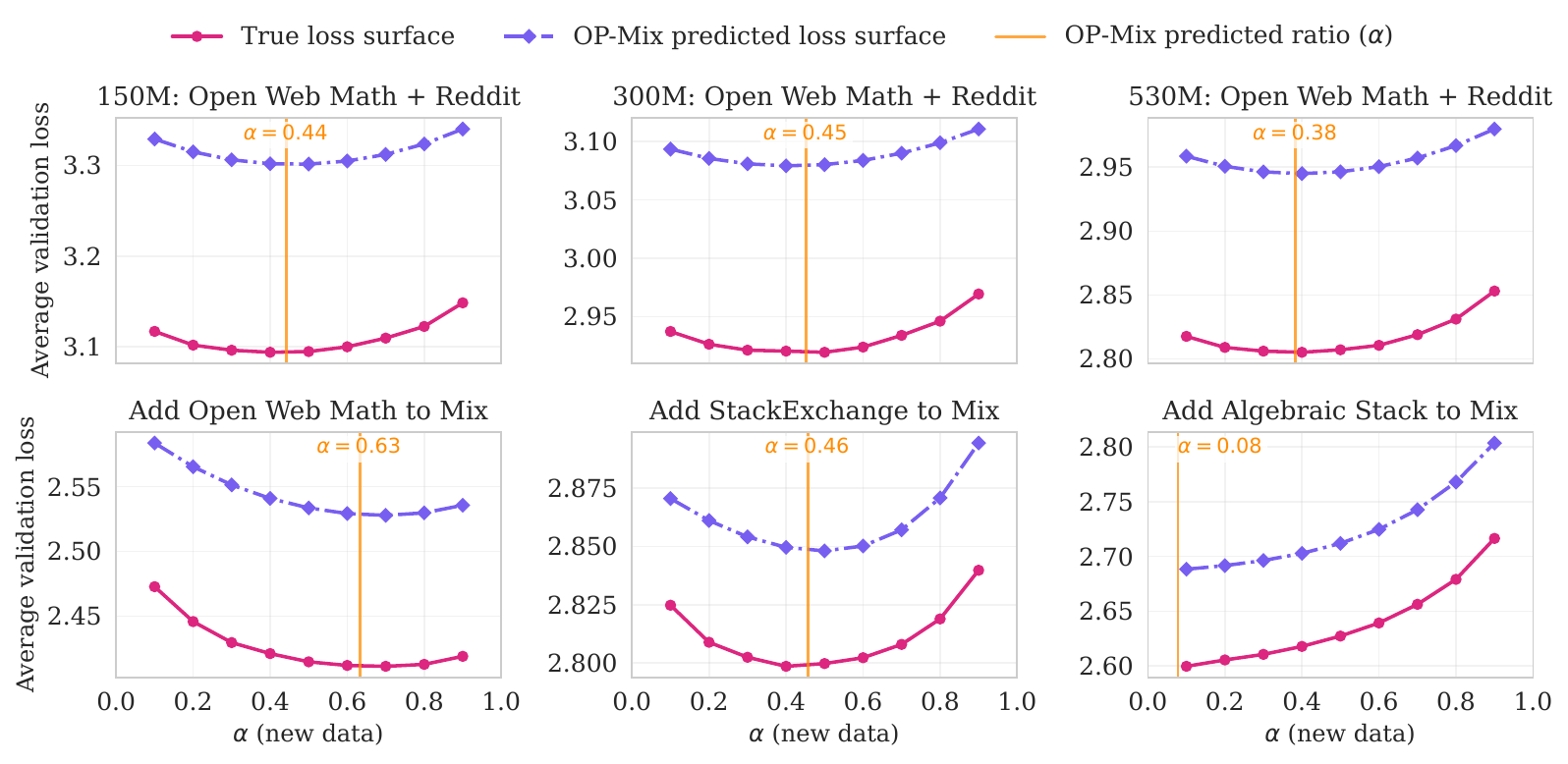}
    \caption{\textbf{\method (purple) closely tracks the true data mixing loss surface (red)}, which we obtain by running full finetuning to completion at each mixture, for different model sizes.}
    \label{fig:optimal}
\end{figure}

\section{Analysis}

\subsection{\method Reliably and Efficiently Estimates Optimal Data Mixtures}
\label{sec:optimality}

We ask whether \method consistently recovers good mixing weights. In Figure \ref{fig:optimal}, we plot the true loss surface with respect to mixture proportions in red and the estimated loss surface from \method in purple. We generate the true loss surface by training on those proportions for all of training, as opposed to training a proxy. We find that merging LoRAs closely tracks the true data mixing loss surface. More concretely, in Figure \ref{fig:regret} in the Appendix, we sweep for the best mixture at each stage in the continual midtraining setting and find that the average increase in loss of \method from the optimal proportions is 0.9\%, compared to a 2.9\% increase for the fixed 10\% data replay baseline. 

\subsection{Theoretical Analysis: Formalizing \method's Sources of Error}
\label{sec:theory}

We now analyze the conditions under which \method recovers an optimal interpolation weight, and bound the suboptimality of its predicted weight to the optimal weight.
This section formalizes the intuition that \method's error is small if 1) LoRA performance is a good approximation of full fine-tuning performance and 2) linear interpolation is a good approximation for mixing. 

\paragraph{Setup.}
Suppose we receive a new domain $D_{m+1}$. Let $F(\alpha) = \frac{1}{N}\sum_{j=1}^{N} f_j(\theta^{\text{train}}(\alpha))$ denote the average evaluation performance of a model trained on the mixture assigning weight $\alpha \in [0,1]$ to $D_{m+1}$ and distributing $1 - \alpha$ over previous domains $D_{\text{old}}$.
\method constructs a proxy for $F$ via two approximations: rather than training a full model on $D_{m+1}$, \method trains two LoRA adapters $\theta^{\text{LoRA}}_{D_{m+1}}$ and $\theta^{\text{LoRA}}_{D_{\text{old}}}$; and rather than training on multiple data mixtures, \method evaluates linear interpolations between the two LoRA adapters, yielding the proxy loss $\hat{F}(\alpha) = \frac{1}{N}\sum_{i=1}^N f_i\!\left(\theta^{\text{LoRA}}(\alpha)\right)$.

To isolate the contributions of these two approximations, we follow \citet{chen2025olmix} and define an intermediate surface $F^M(\alpha) = \frac{1}{N}\sum_{i=1}^N f_i(\theta^{\text{full}}(\alpha))$, where $\theta^{\text{full}}(\alpha) = (1-\alpha) \cdot \theta_{D_{\text{old}}} + \alpha \cdot\theta_{D_{m+1}}$ interpolates full finetuning updates instead of LoRAs. The two errors are then:
\begin{align*}
    \varepsilon_{\text{merge}} := \sup_{\alpha \in [0,1]} \left|F(\alpha) - F^M(\alpha)\right|, \quad \varepsilon_{\text{LoRA}} := \sup_{\alpha \in [0,1]} \left|F^M(\alpha) - \hat{F}(\alpha)\right|.
\end{align*}

If both approximations are exact ($\varepsilon_{\mathrm{merge}} = \varepsilon_{\mathrm{LoRA}} = 0$), then \method returns an optimal interpolation weight. When the approximations are imperfect, the following bound holds:

\begin{remark}[\method performance gap]
\label{thm:gap}
Define the true regularized objective 
$J(\alpha) = F(\alpha) + \lambda \, D_{\text{KL}}\!\bigl(E(\boldsymbol{\alpha}) \,\big\|\, \mu\bigr)$ 
and the proxy objective
$
\hat{J}(\alpha) = \hat{F}(\alpha) + \lambda \, D_{\text{KL}}\!\bigl(E(\boldsymbol{\alpha}) \,\big\|\, \mu\bigr),
$
and let $\alpha^\star \in \arg\min_{\alpha \in [0,1]} J(\alpha)$ and $\hat{\alpha} \in \arg\min_{\alpha \in [0,1]} \hat{J}(\alpha)$. Then:
\begin{align}
    J(\hat{\alpha}) - J(\alpha^\star) \;\le\; 2\!\left(\varepsilon_{\mathrm{merge}} + \varepsilon_{\mathrm{LoRA}}\right).
    \label{eq:gap}
\end{align}
\end{remark}

\begin{proof}[Proof sketch]
See Appendix \ref{app:proofs} for full proof. The following holds:
\begin{align}
    J(\hat{\alpha}) - J(\alpha^\star) \;=\; \underbrace{\big[J(\hat{\alpha}) - \hat{J}(\hat{\alpha})\big]}_{\le\, \varepsilon_{\mathrm{merge}} + \varepsilon_{\mathrm{LoRA}}} \;+\; \underbrace{\big[\hat{J}(\hat{\alpha}) - \hat{J}(\alpha^\star)\big]}_{\le\; 0} \;+\; \underbrace{\big[\hat{J}(\alpha^\star) - J(\alpha^\star)\big]}_{\le\, \varepsilon_{\mathrm{merge}} + \varepsilon_{\mathrm{LoRA}}}. \nonumber
\end{align}
The middle term is nonpositive because $\hat{\alpha}$ also minimizes $\hat{J}$. For the other two terms, the regularization terms cancel and the triangle inequality through $F^M$ gives $|F(\alpha) - \hat{F}(\alpha)| \le \varepsilon_{\mathrm{merge}} + \varepsilon_{\mathrm{LoRA}}$.
\end{proof}
We verify empirically in Figure \ref{fig:regret} that the overall approximation gap is small and non-increasing across continual learning stages.
Furthermore, $\varepsilon_{\text{merge}}$ being small is empirically supported by linear mode connectivity \citep{lmc-lth,pmlr-v162-wortsman22a}, which observes that linear interpolation does not incur large loss spikes  when the finetuned models share a base model (see Corollary~\ref{cor:merge_vanish}). 

\section{Related Work and Discussion}

\paragraph{Continual learning.} The core challenge in continual learning is catastrophic forgetting, where training on new data degrades performance on previously learned tasks~\citep{mccloskey-catastrophic}. Approaches to mitigating forgetting fall into three broad families: regularization-based methods constrain parameter updates to protect knowledge from earlier tasks \citep{kirkpatrick2017overcoming,aljundi2017memory}; replay-based methods retain or regenerate examples from previous tasks \citep{rolnick2019experience,shin2017continual}; and architecture-based methods allocate new capacity for new tasks \citep{rusu2022progressiveneuralnetworks,wang-olora}. \method is a replay-based method. 

In the context of LLMs, forgetting can manifest across pretraining, instruction tuning, and alignment stages~\citep{shi2025continual,zheng2025spurious}. Our work considers in-weights learning, which updates model parameters. A parallel line of work keeps LLM weights frozen and accumulates knowledge in-context, e.g., via soft prompts~\citep{razdaibiedina2023progressive} or modular KV-cache cartridges~\citep{eyuboglu2026cartridges}. However, the storage of in-context approaches grows with the dataset size, so amortizing knowledge from context to parameters using in-weights learning remains relevant. 

\paragraph{Data mixing.} \citet{ye2025data} established scaling laws for data mixtures, showing that downstream performance is a predictable function of mixture proportions. This empirically grounded the pipeline of training small proxy models on candidate mixtures and fitting a regression model to extrapolate to full scale \citep{liu2025regmix,chen2025olmix}. In contrast to this offline approach, several online data mixing algorithms have been proposed for pretraining based on distributionally robust optimization, including DoReMi \citep{xie2023doremi}, DoGE \citep{fan2024doge} and GRAPE \citep{fan2025grape}. \citet{chen2025aioli} showed that both classes of algorithms are instances of the same linear framework.

\paragraph{Limitations and Future Work.} Our experiments top out at 530M parameters for pretraining and midtraining and 7B for instruction tuning, leaving open how \method behaves at frontier scale (70B+ parameters). We also do not characterize how the LoRA proxy and model merging behave as the number of domains grows to 10 or 100. Future work can extend \method to reward-based objectives like RLHF \citep{ouyang2022training}, where LoRA tends to work well \citep{schulman2025lora}. More broadly, our lifecycle-unification direction suggests that other training decisions, such as learning rate schedules or training objectives, may admit similarly unified formulations.

\paragraph{Conclusion.} The various phases of language model training are artificial divisions, and data mixing algorithms should work gracefully across them in a continual setting. Existing methods fall short on two fronts: they cannot incorporate new datasets, and they rely on off-policy proxy models. We address both limitations with \method, the first data mixing algorithm to achieve state-of-the-art results across pretraining, continual midtraining, and continual instruction tuning. By exploiting LoRA and linear mode connectivity to cheaply simulate candidate mixtures, \method turns adaptive data mixing into a lightweight operation that can be repeated whenever new data arrives.

\section*{Acknowledgments}

We thank Graham Neubig, John Langford, Maxime Peyrard, Sebastian Cygert, Mayee Chen, and the NYU Computation and Psycholinguistics Lab for discussion and feedback. MYH is supported by the NSF Graduate Research Fellowship. AG is supported by the Amazon AI PhD Fellowship.

This work was supported in part through the NYU IT High Performance Computing resources, services, and staff expertise. This work was supported by the Institute of Information \& Communications Technology Planning \& Evaluation (IITP) with a grant funded by the Ministry of Science and ICT (MSIT) of the Republic of Korea in connection with the Global AI Frontier Lab International Collaborative Research. This work was also supported by the Samsung Advanced Institute of Technology (under the project Next Generation Deep Learning: From Pattern Recognition to AI) and the National Science Foundation (under NSF Awards 1922658, IIS-2239862, and IIS-2433429).

\bibliographystyle{plainnat}  
\bibliography{references} 

\vfill


\pagebreak
\appendix

\begin{figure}[t]
    \centering
    \includegraphics[width=0.9\linewidth]{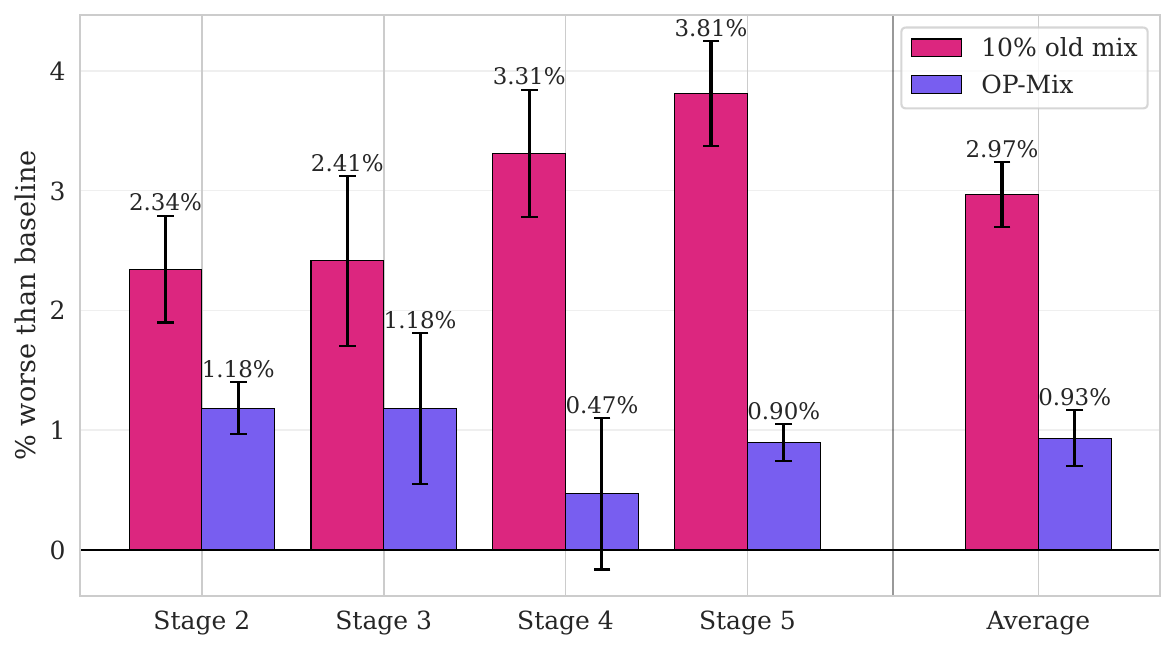}
    \caption{\method versus grid sweep. In the continual midtraining setting, \method consistently achieves regret of 1.18\% or less with respect to the optimal value (estimated by grid sweeping over mixtures). Regret does not grow as more datasets are introduced, unlike with a fixed 10\% old data mixture, where regret does grow.}
    \label{fig:regret}
\end{figure}

\begin{figure}[t]
    \centering
    \includegraphics[width=0.9\linewidth]{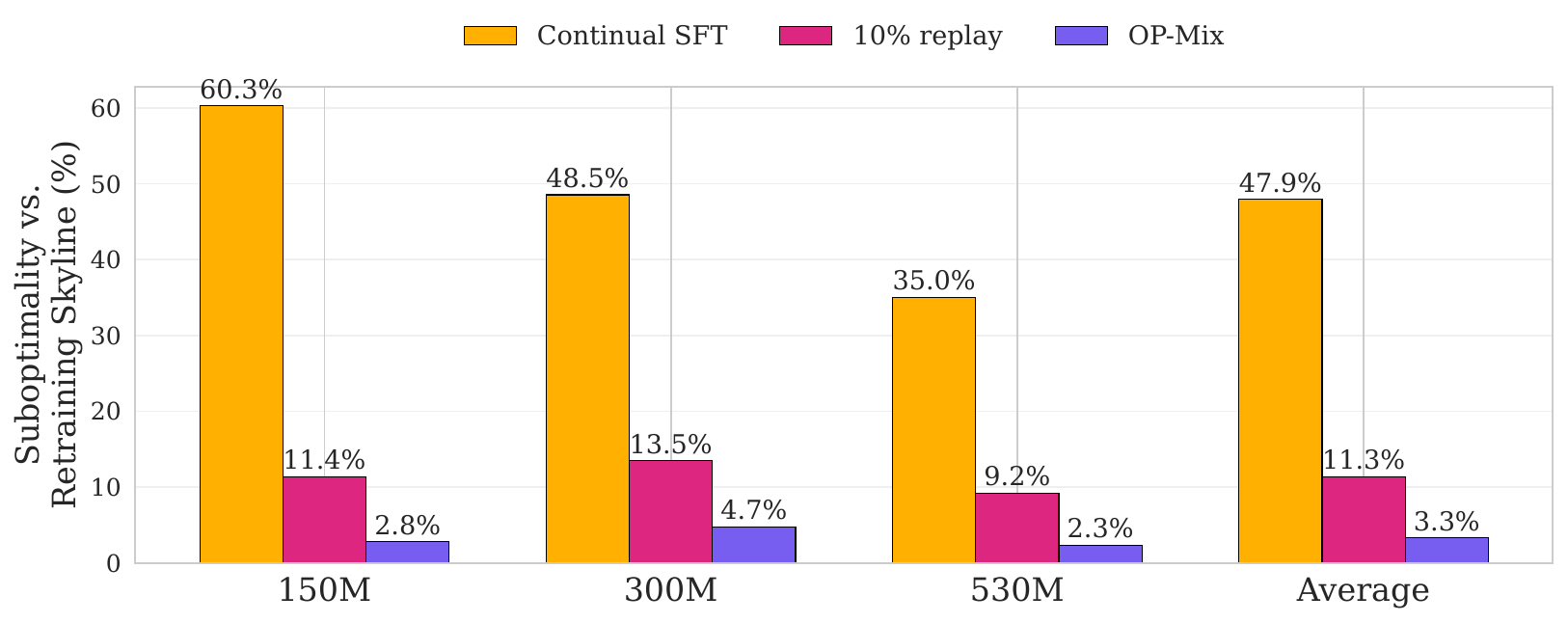}
    \caption{\method versus retraining. In the continual midtraining setting, \method nearly matches the performance of retraining, indicating that it successfully mitigates catastrophic forgetting on previously seen datasets.}
    \label{fig:vs-retraining}
\end{figure}

\begin{figure}[t]
    \centering
    \includegraphics[width=0.9\linewidth]{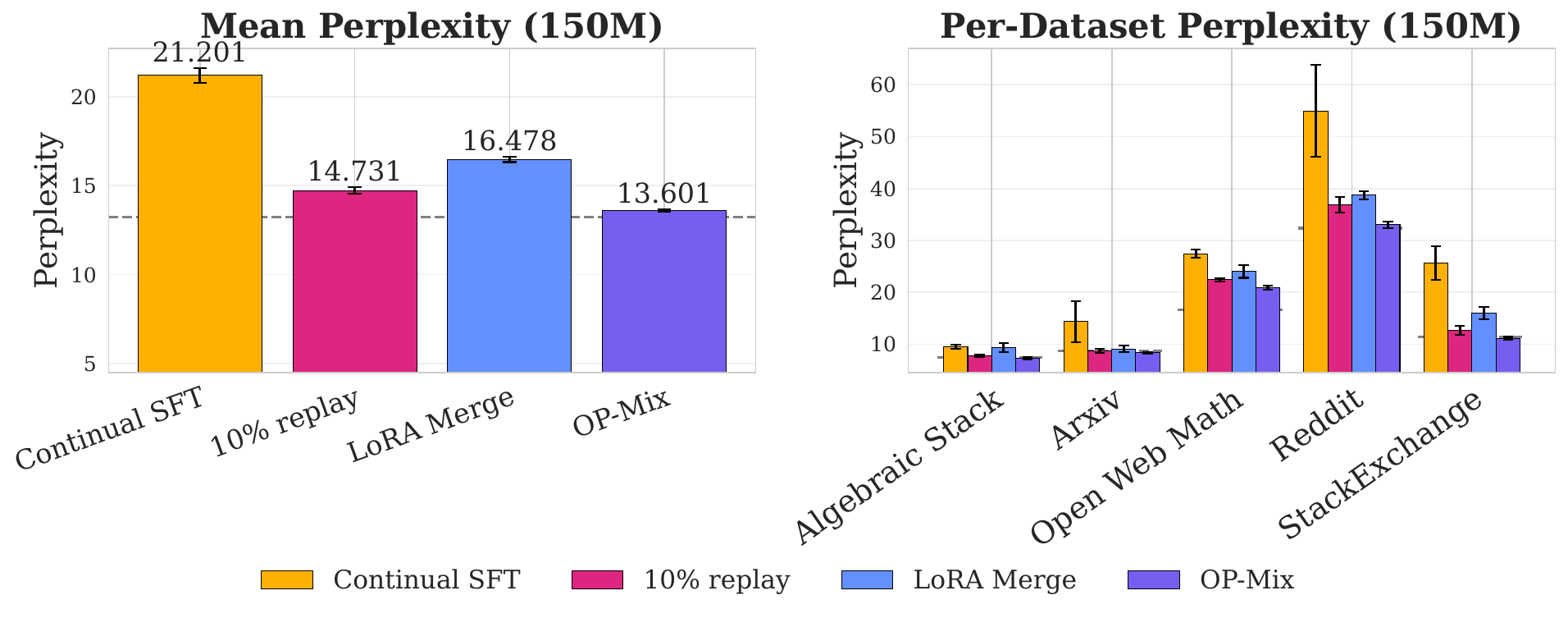}
    \caption{LoRA merging only is not sufficient. Simply merging in trained LoRA adapters (grey) without finetuning underperforms \method (purple).}
    \label{fig:lora-merge}
\end{figure}

\begin{table}[t]
\centering
\small
\setlength{\tabcolsep}{4pt}
\begin{tabular}{l l c c c c c c c c c c}
\toprule
Size & Algorithm & ARC-E & ARC-C & BoolQ & CSQA & HSwag & OBQA & PIQA & WGrd & MMLU & Avg \\
\midrule
\multirow{4}{*}{150M}
 & ERM      & 33.57 & 20.85 & 56.72 & 19.49 & 27.08 & \textbf{24.60} & 58.27 & \textbf{51.04} & 23.01 & 34.96 \\
 & MergeMix & 33.26 & \textbf{21.76} & 60.05 & \textbf{19.52} & 27.37 & 23.93 & 57.85 & 50.72 & 23.00 & 35.28 \\
 & OLMix    & 33.66 & \textbf{21.76} & \textbf{61.27} & 19.49 & 27.53 & 24.00 & \textbf{58.85} & 51.01 & 23.05 & \textbf{35.62} \\
 & OP-Mix   & \textbf{33.68} & 21.16 & 59.92 & 19.49 & \textbf{27.62} & 23.93 & 58.34 & 50.49 & \textbf{23.07} & \textit{35.30} \\
\midrule
\multirow{4}{*}{300M}
 & ERM      & 34.82 & 21.53 & \textbf{60.18} & 19.60 & 28.29 & 25.33 & 59.45 & \textbf{51.33} & 23.06 & \textit{35.96} \\
 & MergeMix & 35.66 & 22.10 & 57.05 & \textbf{19.66} & \textbf{28.98} & 25.40 & 60.01 & 50.20 & \textbf{23.10} & 35.80 \\
 & OLMix    & 35.90 & 21.56 & 54.10 & 19.49 & 28.80 & \textbf{25.80} & 60.39 & 50.59 & 23.04 & 35.52 \\
 & OP-Mix   & \textbf{35.96} & \textbf{22.21} & 57.21 & 19.52 & 28.96 & 25.67 & \textbf{61.37} & 50.93 & 23.06 & \textbf{36.10} \\
\midrule
\multirow{4}{*}{530M}
 & ERM      & 35.24 & 21.96 & 58.83 & 19.52 & 29.03 & 25.33 & 60.12 & 51.43 & 23.03 & 36.06 \\
 & MergeMix & \textbf{36.81} & \textbf{22.30} & 54.76 & \textbf{19.57} & 29.99 & \textbf{26.47} & \textbf{61.95} & 51.35 & \textbf{23.25} & 36.27 \\
 & OLMix    & 36.41 & 21.99 & 59.04 & 19.44 & 29.54 & 26.13 & 61.43 & 50.83 & 23.12 & \textit{36.43} \\
 & OP-Mix   & 36.59 & 21.99 & \textbf{61.45} & \textbf{19.57} & \textbf{30.13} & \textbf{26.47} & 61.15 & \textbf{51.78} & 23.10 & \textbf{36.91} \\
\bottomrule \\
\end{tabular}
\caption{Downstream zero-shot accuracy across ARC-Easy, ARC-Challenge \citep{Clark2018ThinkYH}, BoolQ \citep{clark2019boolq}, CommonsenseQA \citep{talmor-etal-2019-commonsenseqa}, HellaSwag \citep{zellers2019hellaswag}, OpenBookQA \citep{OpenBookQA2018}, PIQA \citep{Bisk2020}, WinoGrande \citep{sakaguchi2019winogrande}, and MMLU \citep{hendryckstest2021}, alongside their unweighted average. We ran evaluations using \texttt{lm-eval-harness} \citep{eval-harness}. Each block reports results for one model size. \textbf{Bold} marks the best result per column within each model size; \textit{italics} mark the second best average score.}
\label{tab:downstream}
\end{table}

\FloatBarrier

\section{Additional Ablations}
\label{app:ablations}

\paragraph{LoRA rank.} We rerun the pretraining proxies with LoRA ranks 8, 32, and 64 (with LoRA scaling $\alpha = 2r$, following \citet{hu2022lora}) and compare the selected mixtures against our default rank of 16 (Table~\ref{tab:lora-rank}). For mixtures $p$ and $q$, we use the symmetric KL divergence $\tfrac{1}{2}[D_{\mathrm{KL}}(p\|q) + D_{\mathrm{KL}}(q\|p)]$. Changing the rank moves the selected mixture by about as much as changing the random seed: cross-rank divergences are of the same order as inter-seed divergences, and the weight placed on old data shifts by at most 2.2 percentage points.

\begin{table}[H]
\centering
\small
\begin{tabular}{@{}l c c c c c@{}}
\toprule
Size & LoRA rank & Seeds & Old-data weight & Inter-seed KL & Cross-rank KL \\
\midrule
\multirow{4}{*}{150M}
 & 8            & 3 & 0.482 & 2.84 & 4.50 \\
 & 16 (default) & 3 & 0.460 & 2.98 & -- \\
 & 32           & 3 & 0.449 & 5.09 & 5.05 \\
 & 64           & 3 & 0.448 & 8.09 & 9.87 \\
\midrule
\multirow{4}{*}{300M}
 & 8            & 3 & 0.509 & 0.81 & 2.58 \\
 & 16 (default) & 3 & 0.513 & 4.46 & -- \\
 & 32           & 3 & 0.496 & 4.23 & 6.74 \\
 & 64           & 2 & 0.501 & 3.39 & 8.38 \\
\midrule
\multirow{4}{*}{530M}
 & 8            & 2 & 0.529 & 2.46 & 3.49 \\
 & 16 (default) & 3 & 0.534 & 3.20 & -- \\
 & 32           & 2 & 0.535 & 2.95 & 4.14 \\
 & 64           & 2 & 0.525 & 17.50 & 10.31 \\
\bottomrule
\end{tabular}
\caption{Sensitivity of the pretraining mixture to LoRA rank. Inter-seed KL is the mean symmetric KL between mixtures selected with different seeds at the same rank. Cross-rank KL is the mean symmetric KL between mixtures selected at that rank and at rank 16, over pairs of different seeds, so that it is directly comparable to inter-seed KL. KL values are scaled by $10^{3}$. Seeds counts the completed seeds among $\{42, 43, 44\}$.}
\label{tab:lora-rank}
\end{table}

\paragraph{LoRA learning rate.} LoRA can be more sensitive to hyperparameters than full finetuning \citep{schulman2025lora}. In continual midtraining at 150M, we halve and double the LoRA learning rate relative to our default of $10^{-3}$ for the first continual stage, and then train the full model on each selected mixture (Table~\ref{tab:lora-lr}). The selected mixtures change by about 5 percentage points, but the resulting losses change by at most 0.001 nats.

\begin{table}[H]
\centering
\small
\begin{tabular}{@{}l c c c c c c@{}}
\toprule
 & \multicolumn{2}{c}{Selected mixture} & \multicolumn{4}{c}{Loss after full training} \\
\cmidrule(lr){2-3} \cmidrule(l){4-7}
LoRA LR & $\alpha_{\text{new}}$ & $|\Delta\alpha_{\text{new}}|$ & Old & New & Mean & $\Delta$ Mean \\
\midrule
$5\times10^{-4}$           & 0.600 & 0.043 & 2.5101 & 2.5556 & 2.5328 & $-0.0010$ \\
$10^{-3}$ (default)        & 0.575 & --    & 2.5076 & 2.5600 & 2.5338 & -- \\
$2\times10^{-3}$           & 0.624 & 0.050 & 2.5139 & 2.5547 & 2.5343 & $+0.0005$ \\
\bottomrule
\end{tabular}
\caption{Sensitivity to LoRA learning rate in continual midtraining (150M, stage 2, averaged over the five dataset orderings). $\alpha_{\text{new}}$ is the selected weight on the new domain; $|\Delta\alpha_{\text{new}}|$ is its mean absolute change from the default learning rate on the same ordering. Losses on old domains, the new domain, and their mean over all domains seen so far are evaluated after full-model training on the selected mixture.}
\label{tab:lora-lr}
\end{table}

\paragraph{Higher replay ratio.} Recent midtraining recipes mix in a larger share of pretraining-style data than 10\% \citep{olmo3}. We therefore compare against 30\% replay for the first continual stage of continual midtraining at 150M (Table~\ref{tab:replay}). Raising the replay ratio from 10\% to 30\% narrows the gap to \method from 1.1\% to 0.3\% in mean loss, but does not close it. \method reduces old-domain loss the most, at a small cost in new-domain loss.

\begin{table}[H]
\centering
\small
\begin{tabular}{@{}l c c c c c c@{}}
\toprule
 & \multicolumn{3}{c}{Old domains} & & & \\
\cmidrule(lr){2-4}
Method & Before & After & Forgetting & New domain & Mean & Gap to \method \\
\midrule
10\% replay & 2.5656 & 2.5918 & $+0.0263$ ($+1.0\%$) & \textbf{2.5324} & 2.5621 & $+1.1\%$ \\
30\% replay & 2.5677 & 2.5310 & $-0.0367$ ($-1.4\%$) & 2.5540 & 2.5425 & $+0.3\%$ \\
\method     & 2.5677 & \textbf{2.5047} & $\mathbf{-0.0630}$ ($-2.5\%$) & 2.5638 & \textbf{2.5343} & -- \\
\bottomrule
\end{tabular}
\caption{Higher replay ratios in continual midtraining (150M, stage 2, averaged over the five dataset orderings). Columns report losses. Forgetting is the change in old-domain loss from the end of stage 1 to the end of stage 2, with the relative change in parentheses; negative values mean old-domain loss improved. Mean averages over all domains seen so far.}
\label{tab:replay}
\end{table}

\paragraph{Locating the optimum.} The \method loss surface in Figure~\ref{fig:optimal} is shifted relative to the true loss surface, so we measure how well its minimizer matches the true one rather than how well the surfaces match pointwise. We fit both surfaces independently with the same log-linear form and KL regularizer used by \method and compare their minimizers (Table~\ref{tab:minima}). Across the six panels, the root-mean-square error in the selected new-data weight is 3.1 percentage points.

\begin{table}[H]
\centering
\small
\begin{tabular}{@{}l l c c c@{}}
\toprule
Panel in Figure~\ref{fig:optimal} & Size & \method optimum & True optimum & Error (pp) \\
\midrule
Open Web Math + Reddit     & 150M & 0.443 & 0.461 & $-1.8$ \\
Open Web Math + Reddit     & 300M & 0.453 & 0.456 & $-0.3$ \\
Open Web Math + Reddit     & 530M & 0.384 & 0.441 & $-5.6$ \\
Add Open Web Math to Mix   & 150M & 0.633 & 0.593 & $+4.0$ \\
Add StackExchange to Mix   & 150M & 0.457 & 0.431 & $+2.7$ \\
Add Algebraic Stack to Mix & 150M & 0.078 & 0.084 & $-0.6$ \\
\bottomrule
\end{tabular}
\caption{Minimizers of the \method-predicted and true loss surfaces in Figure~\ref{fig:optimal}, as the weight on new data. Error is the \method optimum minus the true optimum, in percentage points.}
\label{tab:minima}
\end{table}

\clearpage

\section{Reproducibility}
\label{app:reproducibility}


\subsection{Shared \method\ configuration}
\label{app:shared-config}

Across all three settings, \method\ uses the same high-level structure (Algorithm~\ref{alg:opmix}) and the same regression and solver. Shared choices are listed in Table~\ref{tab:shared-opm}.

\begin{table}[h]
\centering
\small
\begin{tabular}{@{}ll@{}}
\toprule
Setting & Value \\
\midrule
LoRA rank $r$ & 16 \\
LoRA scaling $\alpha$ & 32 \\
LoRA learning rate & $2\times$ base finetuning LR \\
LoRA warmup & 0 steps \\
Regression form & log-linear per eval domain (fit via \texttt{cvxpy}) \\
Mix optimizer & \texttt{cvxpy} \citep{diamond2016cvxpy} on the reduced simplex $\triangle^{K}$ \\
Prior $\mu$ & uniform over the expanded domain set \\
KL regularization $\lambda$ & 0.05 \\
Base dtype & bfloat16 \\
Optimizer (base and LoRA) & AdamW, weight decay $0.01$ \\
LR schedule (base training) & WSD (warmup--stable--decay), 1000 decay steps \\
\bottomrule
\end{tabular}
\caption{Hyperparameters shared by \method\ across pretraining, continual midtraining, and continual instruction tuning.}
\label{tab:shared-opm}
\end{table}

The proxy construction differs slightly between settings. For pretraining, we Dirichlet-sample $P$ proxy mix vectors over $\{D_{\text{old}}, D_{m+1}, \dots, D_{m+K}\}$. For continual midtraining and continual instruction tuning, where $K=1$ new domain arrives per stage, we replace Dirichlet sampling with a deterministic 9-point grid over the old/new axis at $\alpha_{\text{new}} \in \{0.1, 0.2, \dots, 0.9\}$. In both continual settings the old- and new-domain LoRAs are trained with a \emph{10/90} split (old probe mixes 10\% of the new domain into the old mix; new probe mixes 10\% of the old mix into the new domain) rather than one-hot specialization; we found this to prevent overestimation of forgetting while still being mathematically correct.

\subsection{Pretraining}
\label{app:pretrain}

We pretrain from configuration-initialized OLMo models at three sizes on a five-domain mix of Algebraic Stack, ArXiv, c4, Reddit, and StackExchange, all tokenized with the DataDecide Dolma v1.5 tokenizer \cite{magnusson2025datadecide}. The model ladder uses the \texttt{allenai/DataDecide-c4-\{150M, 300M, 530M\}} architectures, initialized from config only (random weights). Per-size hyperparameters are in Table~\ref{tab:pretrain-hp}.

\begin{table}[h]
\centering
\small
\begin{tabular}{@{}lccc@{}}
\toprule
 & 150M & 300M & 530M \\
\midrule
Base architecture & DataDecide-c4-150M & DataDecide-c4-300M & DataDecide-c4-530M \\
Total steps $R$ & 50{,}000 & 50{,}000 & 80{,}000 \\
Global batch size & 32 & 64 & 64 \\
Sequence length & 1024 & 1024 & 1024 \\
Learning rate & $5\times10^{-4}$ & $4\times10^{-4}$ & $3\times10^{-4}$ \\
Warmup steps & 1{,}000 & 5{,}000 & 8{,}000 \\
ERM warm start & 0.20$R$ & 0.20$R$ & 0.20$R$ \\
LoRA steps & 5{,}000 & 5{,}000 & 5{,}000 \\
Proxy points $P$ & 20 & 20 & 20 \\
\bottomrule
\end{tabular}
\caption{Pretraining hyperparameters for \method\ and the ERM baseline. MergeMix uses the same prefix, proxy count, and probe length, but trains full-parameter proxies instead of LoRAs. OLMix uses a separate 20M-parameter proxy model (Table~\ref{tab:proxy-baseline}).}
\label{tab:pretrain-hp}
\end{table}

The ERM prefix is trained on the uniform mix $(1/5, 1/5, 1/5, 1/5, 1/5)$. After the prefix, each of the 5 LoRA probes is trained on a \emph{90\%-on-its-domain / 10\%-on-the-old-mix} partition so the span of the 5 adapters covers the full simplex interior. We then build $P=20$ Dirichlet-sampled interpolation merges, evaluate each on the held-out shards of all 5 domains, fit the log-linear regression, and train for the remaining $0.8R$ steps on the fitted mix.


\begin{table}[h]
\centering
\small
\begin{tabular}{@{}ll@{}}
\toprule
OLMix proxy (20M) & Value \\
\midrule
Base architecture & DataDecide-c4-20M \\
Total steps & 10{,}000 \\
Batch size & 32 \\
Learning rate & $1\times10^{-3}$ \\
Warmup steps & 5{,}000 \\
Proxy mixtures & 64 Dirichlet samples \\
\bottomrule
\end{tabular}
\caption{Proxy configuration for the OLMix baseline in pretraining.}
\label{tab:proxy-baseline}
\end{table}

\subsection{Continual midtraining}
\label{app:continual}

We start from the pretrained DataDecide-c4 checkpoints from pretraining and continually finetune on the same five domains as pretraining, introduced one stage at a time. To control for ordering effects we run all five cyclic permutations \texttt{ord0} through \texttt{ord4} (Table~\ref{tab:ordering}).

\begin{table}[h]
\centering
\small
\begin{tabular}{@{}ll@{}}
\toprule
Ordering & Stage sequence \\
\midrule
\texttt{ord0} & algebraic\_stack $\to$ arxiv $\to$ open\_web\_math $\to$ reddit $\to$ stackexchange \\
\texttt{ord1} & arxiv $\to$ open\_web\_math $\to$ reddit $\to$ stackexchange $\to$ algebraic\_stack \\
\texttt{ord2} & open\_web\_math $\to$ reddit $\to$ stackexchange $\to$ algebraic\_stack $\to$ arxiv \\
\texttt{ord3} & reddit $\to$ stackexchange $\to$ algebraic\_stack $\to$ arxiv $\to$ open\_web\_math \\
\texttt{ord4} & stackexchange $\to$ algebraic\_stack $\to$ arxiv $\to$ open\_web\_math $\to$ reddit \\
\bottomrule
\end{tabular}
\caption{Cyclic orderings of the five midtraining domains. Results are averaged across these five orderings.}
\label{tab:ordering}
\end{table}

Per-stage hyperparameters are in Table~\ref{tab:continual-hp}. Every stage consists of two LoRA probes (old and new, with the 10/90 split above), a 9-point 1-D proxy scan, a \texttt{cvxpy} mix fit on the reduced $(\alpha_{\text{old}}, \alpha_{\text{new}})$ simplex, expansion of $\alpha_{\text{old}}^\star$ onto the previous stage's mix, and a full-model finetune on the expanded mix.

\begin{table}[h]
\centering
\small
\begin{tabular}{@{}lccc@{}}
\toprule
 & 150M & 300M & 530M \\
\midrule
Base model & DataDecide-c4-150M & DataDecide-c4-300M & DataDecide-c4-530M \\
Steps per stage $R$ & 10{,}000 & 12{,}500 & 15{,}000 \\
Batch size & 32 & 32 & 32 \\
Sequence length & 1024 & 1024 & 1024 \\
Learning rate & $5\times10^{-4}$ & $4\times10^{-4}$ & $3\times10^{-4}$ \\
Warmup steps (stage 1) & 1{,}000 & 1{,}250 & 1{,}500 \\
Warmup steps (stage $k\ge 2$) & 0 & 0 & 0 \\
LoRA steps & 2{,}500 & 2{,}500 & 2{,}500 \\
Proxy grid & $\alpha_{\text{new}}\in\{0.1,\dots,0.9\}$ & $\alpha_{\text{new}}\in\{0.1,\dots,0.9\}$ & $\alpha_{\text{new}}\in\{0.1,\dots,0.9\}$ \\
Old probe new-weight & 0.1 & 0.1 & 0.1 \\
New probe new-weight & 0.9 & 0.9 & 0.9 \\
\bottomrule
\end{tabular}
\caption{Per-stage hyperparameters for continual midtraining. Stage 1 is a single-domain finetune; stages 2--5 run the full \method\ pipeline on top of the previous stage's checkpoint.}
\label{tab:continual-hp}
\end{table}

Baselines inherit the same $R$, batch size, sequence length, learning rate, and warmup schedule. The ``10\% data replay'' baseline fixes $\alpha_{\text{old}}=0.1$ at every stage and expands onto the previous mix using the same expansion map $E$ as \method. ``Retrain'' trains from the original base model for $k\cdot R$ steps on the uniform mix over the $k$ domains seen so far.

\subsection{Continual instruction tuning}
\label{app:sdft}

We use Qwen2.5-7B-Instruct \citep{yang2024qwen2} as the base, and reuse the three domains and ordering of \citet{shenfeld2026selfdistillation}: Tool Use (4{,}046 examples) $\to$ Science (1{,}233 examples) $\to$ Medical (10{,}000 examples). Each stage is one epoch over its dataset (capped at 10{,}000 examples). We evaluate with the SDFT accuracy metric of \citet{shenfeld2026selfdistillation}, averaged across the domains seen so far.

\begin{table}[h]
\centering
\small
\begin{tabular}{@{}ll@{}}
\toprule
 & Value \\
\midrule
Base model & Qwen2.5-7B-Instruct \\
Num train epochs (full finetune) & 1 \\
Max train samples per stage & 10{,}000 \\
Per-device train batch size & 4 \\
Prompts per SDFT batch & 32 \\
Learning rate (full finetune) & $2\times10^{-5}$ \\
Learning rate (LoRA probe) & $4\times10^{-5}$ \\
LoRA probe max steps & 256 \\
SDFT ref-model mixup $\alpha$ & 0.01 \\
SDFT $\beta$ & 0.0 \\
Proxy grid & $\alpha_{\text{new}}\in\{0.1,\dots,0.9\}$ \\
GPU constraint & H200 \\
\bottomrule
\end{tabular}
\caption{Hyperparameters for the Qwen2.5-7B-Instruct continual instruction tuning experiments. The same settings are used for both the SFT and SDFT variants; SFT simply disables the SDFT-specific options. The LoRA probe is trained for a fixed 256 optimizer steps (short relative to the $\sim\!2{,}500$-step midtraining LoRA) because the instruction datasets here are small.}
\label{tab:sdft-hp}
\end{table}

The  difference between SFT and SDFT \citep{shenfeld2026selfdistillation} is the training objective: SFT uses cross-entropy against the dataset targets, while SDFT replaces the targets with reverse KL-divergence to a teacher model's distribution. In the case of SDFT, the teacher is a moving average variant of the student that also receives the correct answer. In any case, \method\ is applied identically on top of either objective: it only chooses the data-mix weights fed into the training loop, so ``SFT + \method'' and ``SDFT + \method'' use exactly the proxy, regression, and solver pipeline of Table~\ref{tab:shared-opm}, just with the respective loss.

\paragraph{Compute.} All experiments run on a cluster with a mix of A100, H100, L40S, and H200 GPUs. A nice property of LoRA is that it allows on-policy proxies to run on heterogeneous compute: for example, we can run pretraining on an H200 but run proxies on an L40S, which has less than a third of the GPU VRAM.



\paragraph{Seeds and variance.} Continual midtraining configuration is run for a single seed per \{model size, ordering\} cell; variance in the continual setting is instead quantified across the five cyclic orderings. Pretraining and continual instruction tuning results are both averaged across seeds $s \in \{42, 43, 44\}$.

\section{Theory}
\label{app:proofs}

Consider one continual step of Algorithm~\ref{alg:opmix}. The current stage has previous domains $\{D_1,\dots,D_m\}$, over which we played the mixture $p_{t-1}\in\triangle^{m-1}$, and receives new domains $\{D_{m+1},\dots,D_{m+K}\}$. Let
\[
\boldsymbol{\alpha} = (\alpha_{\text{old}}, \alpha_{m+1}, \dots, \alpha_{m+K}) \in \triangle^{K}
\]
denote the reduced-simplex weights used by \method, and let $E:\triangle^K \to \triangle^{m+K-1}$ be the mixture expansion map from Algorithm~\ref{alg:opmix}. We write $\theta_{\text{base}}$ for the current base model, and $\theta^{\text{train}}(\boldsymbol{\alpha})$ for the model obtained by continuing training from $\theta_{\text{base}}$ on the expanded mixture $E(\boldsymbol{\alpha})$.

For the proxy construction, let $\theta^{\text{full}}_{\text{old}}$ be the result of full finetuning on the old-data mixture $p_{t-1}$, and let $\theta^{\text{full}}_{D_{m+k}}$ be the result of full finetuning on $D_{m+k}$, all starting from $\theta_{\text{base}}$. Likewise, let $\theta^{\text{LoRA}}_{\text{old}}$ and $\theta^{\text{LoRA}}_{D_{m+k}}$ be the corresponding LoRA-adapted models produced by Algorithm~\ref{alg:opmix}, with the adapters applied to $\theta_{\text{base}}$. We then define the merged full-model proxy and merged LoRA proxy:
\begin{align}
    \theta^M(\boldsymbol{\alpha}) &:= \alpha_{\text{old}} \, \theta^{\text{full}}_{\text{old}} + \sum_{k=1}^{K} \alpha_{m+k} \, \theta^{\text{full}}_{D_{m+k}}, \\
    \hat{\theta}(\boldsymbol{\alpha}) &:= \alpha_{\text{old}} \, \theta^{\text{LoRA}}_{\text{old}} + \sum_{k=1}^{K} \alpha_{m+k} \, \theta^{\text{LoRA}}_{D_{m+k}}.
\end{align}
Because the coefficients in $\boldsymbol{\alpha}$ sum to one and every model above is trained from the same $\theta_{\text{base}}$, these expressions are convex combinations of the corresponding parameter updates.

The three loss surfaces are therefore
\begin{align}
    F(\boldsymbol{\alpha}) &:= \frac{1}{N}\sum_{j=1}^{N} f_j\!\left(\theta^{\text{train}}(\boldsymbol{\alpha})\right), \\
    F^M(\boldsymbol{\alpha}) &:= \frac{1}{N}\sum_{j=1}^{N} f_j\!\left(\theta^M(\boldsymbol{\alpha})\right), \\
    \hat{F}(\boldsymbol{\alpha}) &:= \frac{1}{N}\sum_{j=1}^{N} f_j\!\left(\hat{\theta}(\boldsymbol{\alpha})\right).
\end{align}

The regularized objectives are
\begin{align}
    J(\boldsymbol{\alpha}) &:= F(\boldsymbol{\alpha}) + \lambda \, D_{\text{KL}}\!\bigl(E(\boldsymbol{\alpha}) \,\big\|\, \mu\bigr), \\
    \hat{J}(\boldsymbol{\alpha}) &:= \hat{F}(\boldsymbol{\alpha}) + \lambda \, D_{\text{KL}}\!\bigl(E(\boldsymbol{\alpha}) \,\big\|\, \mu\bigr),
\end{align}
with optimizers
\begin{align}
    \boldsymbol{\alpha}^\star &:= \arg\min_{\boldsymbol{\alpha} \in \triangle^K} J(\boldsymbol{\alpha}), \\
    \hat{\boldsymbol{\alpha}} &:= \arg\min_{\boldsymbol{\alpha} \in \triangle^K} \hat{J}(\boldsymbol{\alpha}).
\end{align}

The two approximation errors are
\begin{align}
    \varepsilon_{\text{merge}} &:= \sup_{\boldsymbol{\alpha} \in \triangle^K} \left|F(\boldsymbol{\alpha}) - F^M(\boldsymbol{\alpha})\right|,\\
    \varepsilon_{\text{LoRA}} &:= \sup_{\boldsymbol{\alpha} \in \triangle^K} \left|F^M(\boldsymbol{\alpha}) - \hat{F}(\boldsymbol{\alpha})\right|.
\end{align}

\begin{assumption}[Idealized proxy optimization]
\label{assump:exact}
We assume: (i) the fitted regression surface used in Algorithm~\ref{alg:opmix} recovers the merged-LoRA loss surface exactly, so the optimization step is equivalent to minimizing $\hat{J}$ over $\triangle^K$; and (ii) both $J$ and $\hat{J}$ are minimized exactly. This isolates the structural approximation errors induced by LoRA and model merging from finite-sample regression error, numerical optimization error, and proxy-training budget mismatch.
\end{assumption}

\subsection{Exact Recovery}

\begin{proposition}[Exact recovery]
\label{prop:exact}
Under Assumption~\ref{assump:exact}, if $\varepsilon_{\mathrm{merge}} = \varepsilon_{\mathrm{LoRA}} = 0$, then $\hat{F} = F$ on $\triangle^K$ and every minimizer of $\hat{J}$ is also a minimizer of $J$. In particular,
\[
\hat{\boldsymbol{\alpha}} \in \arg\min_{\boldsymbol{\alpha} \in \triangle^K} J(\boldsymbol{\alpha}).
\]
\end{proposition}

\begin{proof}
If $\varepsilon_{\mathrm{merge}} = 0$, then $F(\boldsymbol{\alpha}) = F^M(\boldsymbol{\alpha})$ for all $\boldsymbol{\alpha} \in \triangle^K$, so the merged full-model proxy matches the loss of training on the expanded mixture. If additionally $\varepsilon_{\mathrm{LoRA}} = 0$, then $F^M(\boldsymbol{\alpha}) = \hat{F}(\boldsymbol{\alpha})$ for all $\boldsymbol{\alpha}$, so the merged LoRA proxy is also exact. Hence $\hat{F}(\boldsymbol{\alpha}) = F(\boldsymbol{\alpha})$ on $\triangle^K$, which implies $\hat{J}(\boldsymbol{\alpha}) = J(\boldsymbol{\alpha})$ because both objectives share the same regularizer. Therefore $\arg\min \hat{J} = \arg\min J$, and in particular any optimizer $\hat{\boldsymbol{\alpha}}$ of the proxy objective is also an optimizer of the true objective.
\end{proof}

This is the analog of Lemma~2 of \citet{chen2025olmix}, which shows exact recovery when the reused mixture is itself optimal. For \method, the corresponding ideal condition is that the reduced-simplex proxy surface exactly matches the true objective after expansion by $E$.

\subsection{Performance Gap Bound}

We first establish a uniform approximation lemma, then use it to prove Theorem~\ref{thm:gap}.

\begin{lemma}[Uniform proxy error]
\label{lem:uniform}
For any $\boldsymbol{\alpha} \in \triangle^K$:
\begin{align}
    \left|F(\boldsymbol{\alpha}) - \hat{F}(\boldsymbol{\alpha})\right| \;\le\; \varepsilon_{\mathrm{merge}} + \varepsilon_{\mathrm{LoRA}}.
\end{align}
\end{lemma}

\begin{proof}
By the triangle inequality, introducing the intermediate surface $F^M$:
\begin{align}
    \left|F(\boldsymbol{\alpha}) - \hat{F}(\boldsymbol{\alpha})\right| &= \left|F(\boldsymbol{\alpha}) - F^M(\boldsymbol{\alpha}) + F^M(\boldsymbol{\alpha}) - \hat{F}(\boldsymbol{\alpha})\right| \nonumber \\
    &\le \left|F(\boldsymbol{\alpha}) - F^M(\boldsymbol{\alpha})\right| + \left|F^M(\boldsymbol{\alpha}) - \hat{F}(\boldsymbol{\alpha})\right| \nonumber \\
    &\le \sup_{\boldsymbol{\alpha}' \in \triangle^K} \left|F(\boldsymbol{\alpha}') - F^M(\boldsymbol{\alpha}')\right| + \sup_{\boldsymbol{\alpha}' \in \triangle^K} \left|F^M(\boldsymbol{\alpha}') - \hat{F}(\boldsymbol{\alpha}')\right| \nonumber \\
    &= \varepsilon_{\mathrm{merge}} + \varepsilon_{\mathrm{LoRA}}. \nonumber \qedhere
\end{align}
\end{proof}

\setcounter{savedtheorem}{\value{theorem}}                                                
\setcounter{theorem}{0}

\begin{remark}[\method performance gap, full proof]
Under Assumption~\ref{assump:exact}:
\begin{align}
    J(\hat{\boldsymbol{\alpha}}) - J(\boldsymbol{\alpha}^\star) \;\le\; 2\!\left(\varepsilon_{\mathrm{merge}} + \varepsilon_{\mathrm{LoRA}}\right). \tag{\ref{eq:gap}}
\end{align}
\end{remark}

\setcounter{theorem}{\value{savedtheorem}}

\begin{proof}
We decompose the objective gap by adding and subtracting $\hat{J}$:
\begin{align}
    J(\hat{\boldsymbol{\alpha}}) - J(\boldsymbol{\alpha}^\star) &= \big[J(\hat{\boldsymbol{\alpha}}) - \hat{J}(\hat{\boldsymbol{\alpha}})\big] + \big[\hat{J}(\hat{\boldsymbol{\alpha}}) - \hat{J}(\boldsymbol{\alpha}^\star)\big] + \big[\hat{J}(\boldsymbol{\alpha}^\star) - J(\boldsymbol{\alpha}^\star)\big]. \label{eq:decomp}
\end{align}

\textbf{Middle term.} Since $\hat{\boldsymbol{\alpha}} = \arg\min_{\boldsymbol{\alpha} \in \triangle^K} \hat{J}(\boldsymbol{\alpha})$ and $\boldsymbol{\alpha}^\star \in \triangle^K$:
\begin{align}
    \hat{J}(\hat{\boldsymbol{\alpha}}) - \hat{J}(\boldsymbol{\alpha}^\star) \;\le\; 0. \label{eq:middle}
\end{align}

\textbf{First term.} Note that $J(\boldsymbol{\alpha}) - \hat{J}(\boldsymbol{\alpha}) = F(\boldsymbol{\alpha}) - \hat{F}(\boldsymbol{\alpha})$ for any $\boldsymbol{\alpha}$, since the regularization terms are identical in $J$ and $\hat{J}$. Applying Lemma~\ref{lem:uniform}:
\begin{align}
    J(\hat{\boldsymbol{\alpha}}) - \hat{J}(\hat{\boldsymbol{\alpha}}) = F(\hat{\boldsymbol{\alpha}}) - \hat{F}(\hat{\boldsymbol{\alpha}}) \;\le\; \left|F(\hat{\boldsymbol{\alpha}}) - \hat{F}(\hat{\boldsymbol{\alpha}})\right| \;\le\; \varepsilon_{\mathrm{merge}} + \varepsilon_{\mathrm{LoRA}}. \label{eq:first}
\end{align}

\textbf{Third term.} By the same reasoning:
\begin{align}
    \hat{J}(\boldsymbol{\alpha}^\star) - J(\boldsymbol{\alpha}^\star) = \hat{F}(\boldsymbol{\alpha}^\star) - F(\boldsymbol{\alpha}^\star) \;\le\; \left|\hat{F}(\boldsymbol{\alpha}^\star) - F(\boldsymbol{\alpha}^\star)\right| \;\le\; \varepsilon_{\mathrm{merge}} + \varepsilon_{\mathrm{LoRA}}. \label{eq:third}
\end{align}

Substituting \eqref{eq:middle}, \eqref{eq:first}, and \eqref{eq:third} into \eqref{eq:decomp}:
\begin{align}
    J(\hat{\boldsymbol{\alpha}}) - J(\boldsymbol{\alpha}^\star) &\;\le\; (\varepsilon_{\mathrm{merge}} + \varepsilon_{\mathrm{LoRA}}) + 0 + (\varepsilon_{\mathrm{merge}} + \varepsilon_{\mathrm{LoRA}}) \nonumber \\
    &\;=\; 2(\varepsilon_{\mathrm{merge}} + \varepsilon_{\mathrm{LoRA}}). \nonumber \qedhere
\end{align}
\end{proof}

\subsection{Characterizing the Approximation Errors}

The bound in Theorem~\ref{thm:gap} reduces the analysis of \method to bounding $\varepsilon_{\mathrm{merge}}$ and $\varepsilon_{\mathrm{LoRA}}$ separately. We now provide Lipschitz-based characterizations of each.

\begin{proposition}[LoRA approximation bound]
\label{prop:lora_lip}
If each downstream metric $f_j$ is $L_j$-Lipschitz in model parameters with respect to the Frobenius norm, i.e., $|f_j(\theta) - f_j(\theta')| \le L_j \|\theta - \theta'\|_F$ for all $\theta, \theta'$, then with $L = \frac{1}{N}\sum_{j=1}^N L_j$:
\begin{align}
    \varepsilon_{\mathrm{LoRA}} \;\le\; L \cdot \max\!\Biggl\{
        \left\|\theta^{\mathrm{full}}_{\mathrm{old}} - \theta^{\mathrm{LoRA}}_{\mathrm{old}}\right\|_F,\;
        \max_{k \in [K]} \left\|\theta^{\mathrm{full}}_{D_{m+k}} - \theta^{\mathrm{LoRA}}_{D_{m+k}}\right\|_F
    \Biggr\}.
    \label{eq:lora_lip_app}
\end{align}
\end{proposition}

\begin{proof}
For any $\boldsymbol{\alpha} \in \triangle^K$:
\begin{align}
    \left|F^M(\boldsymbol{\alpha}) - \hat{F}(\boldsymbol{\alpha})\right| &= \left|\frac{1}{N}\sum_{j=1}^N \Big[f_j\!\big(\theta^M(\boldsymbol{\alpha})\big) - f_j\!\big(\hat{\theta}(\boldsymbol{\alpha})\big)\Big]\right| \nonumber \\
    &\le \frac{1}{N}\sum_{j=1}^N L_j \left\|\theta^M(\boldsymbol{\alpha}) - \hat{\theta}(\boldsymbol{\alpha})\right\|_F \nonumber \\
    &= L \left\|\alpha_{\mathrm{old}} \Big(\theta^{\mathrm{full}}_{\mathrm{old}} - \theta^{\mathrm{LoRA}}_{\mathrm{old}}\Big) + \sum_{k=1}^K \alpha_{m+k} \Big(\theta^{\mathrm{full}}_{D_{m+k}} - \theta^{\mathrm{LoRA}}_{D_{m+k}}\Big)\right\|_F \nonumber \\
    &\le L \left[\alpha_{\mathrm{old}} \left\|\theta^{\mathrm{full}}_{\mathrm{old}} - \theta^{\mathrm{LoRA}}_{\mathrm{old}}\right\|_F + \sum_{k=1}^K \alpha_{m+k} \left\|\theta^{\mathrm{full}}_{D_{m+k}} - \theta^{\mathrm{LoRA}}_{D_{m+k}}\right\|_F \right] \nonumber \\
    &\le L \cdot \max\!\Biggl\{
        \left\|\theta^{\mathrm{full}}_{\mathrm{old}} - \theta^{\mathrm{LoRA}}_{\mathrm{old}}\right\|_F,\;
        \max_{k \in [K]} \left\|\theta^{\mathrm{full}}_{D_{m+k}} - \theta^{\mathrm{LoRA}}_{D_{m+k}}\right\|_F
    \Biggr\}. \qedhere 
\end{align}
\end{proof}

\begin{proposition}[Merging approximation bound]
\label{prop:merge_lip}
Under the same Lipschitz condition as Proposition~\ref{prop:lora_lip}, let $\theta^{\mathrm{train}}(\boldsymbol{\alpha})$ denote the model trained on the expanded mixture $E(\boldsymbol{\alpha})$, and let $\theta^M(\boldsymbol{\alpha})$ denote the merged full-model proxy. Then:
\begin{align}
    \varepsilon_{\mathrm{merge}} \;\le\; L \cdot \sup_{\boldsymbol{\alpha} \in \triangle^K} \left\|\theta^{\mathrm{train}}(\boldsymbol{\alpha}) - \theta^M(\boldsymbol{\alpha})\right\|_F.
    \label{eq:merge_lip_app}
\end{align}
\end{proposition}

\begin{proof}
For any $\boldsymbol{\alpha} \in \triangle^K$:
\begin{align}
    \left|F(\boldsymbol{\alpha}) - F^M(\boldsymbol{\alpha})\right| &= \left|\frac{1}{N}\sum_{j=1}^N \Big[f_j\!\big(\theta^{\text{train}}(\boldsymbol{\alpha})\big) - f_j\!\big(\theta^M(\boldsymbol{\alpha})\big)\Big]\right| \nonumber \\
    &\le \frac{1}{N}\sum_{j=1}^N L_j \left\|\theta^{\text{train}}(\boldsymbol{\alpha}) - \theta^M(\boldsymbol{\alpha})\right\|_F \nonumber \\
    &= L \left\|\theta^{\text{train}}(\boldsymbol{\alpha}) - \theta^M(\boldsymbol{\alpha})\right\|_F. \nonumber
\end{align}
Taking the supremum over $\boldsymbol{\alpha}$ yields the result.
\end{proof}

\begin{corollary}[Merging error vanishes under linear mode connectivity]
\label{cor:merge_vanish}
Define the linearity gap
\[
\delta_{\mathrm{LMC}} := \sup_{\boldsymbol{\alpha} \in \triangle^K} \left\|\theta^{\mathrm{train}}(\boldsymbol{\alpha}) - \theta^M(\boldsymbol{\alpha})\right\|_F.
\]
For any $\varepsilon > 0$, if $\delta_{\mathrm{LMC}} \le \varepsilon / L$, then $\varepsilon_{\mathrm{merge}} \le \varepsilon$.
\end{corollary}

\begin{proof}
$\varepsilon_{\mathrm{merge}} \le L \cdot \delta_{\mathrm{LMC}} \le L \cdot \varepsilon / L = \varepsilon$.
\end{proof}

The linearity gap $\delta_{\mathrm{LMC}}$ measures how well the convex hull of the old-data model and the new-domain models approximates actual training on the expanded mixture. Because \method compresses all historical data into the single component $\theta^{\mathrm{full}}_{\mathrm{old}}$, it only needs this reduced simplex to be well behaved, rather than requiring one separately trained model for every historical domain.

Linear mode connectivity says that moving along the interpolation path between endpoint models does not create a large loss barrier. That observation is exactly why model merging is a plausible proxy in \method: it suggests that linear interpolation should not cause catastrophic loss blowups. 



\end{document}